\documentclass{article}
\usepackage[table]{xcolor}

\usepackage{microtype}
\usepackage{graphicx}
\usepackage{subfigure}
\usepackage{booktabs} 

\usepackage{hyperref}

\usepackage[accepted]{icml2024}

\usepackage{amsmath}
\usepackage{amssymb}
\usepackage{mathtools}
\usepackage{amsthm}

\usepackage[capitalize,noabbrev]{cleveref}

\theoremstyle{plain}
\newtheorem{theorem}{Theorem}[section]
\newtheorem{proposition}[theorem]{Proposition}
\newtheorem{lemma}[theorem]{Lemma}

\theoremstyle{definition}

\theoremstyle{remark}
\newtheorem{remark}[theorem]{Remark}

\usepackage[textsize=tiny]{todonotes}

\usepackage{bbm}
\usepackage{amssymb}
\usepackage{pifont}
\usepackage{arydshln}

\definecolor{LightCyan}{rgb}{0.88,1,1}

\def\ie{{\em i.e.},\ }
\def\eg{{\em e.g.},\ }

\def\Surv#1#2{S(#1 \mid #2)}
\def\SurvPred#1#2{\hat{S}(#1 \mid #2)}
\def\Bfx#1{\boldsymbol{x}_{#1}}
\def\E{\mathbb{E}}
\def\prob{\mathbb{P}}
\def\KM{\text{KM}}
\def\Data{\mathcal{D}}
\def\Model{\mathcal{M}}
\def\Q{\text{Quantile}}

\usepackage[inline]{trackchanges}
\addeditor{RG}
\addeditor{SQ}

\icmltitlerunning{Conformalized Survival Distributions}

\begin{document}

\twocolumn[
\icmltitle{Conformalized Survival Distributions:\\ A Generic Post-Process to Increase Calibration}




\begin{icmlauthorlist}
\icmlauthor{Shi-ang Qi}{cs}
\icmlauthor{Yakun Yu}{ece}
\icmlauthor{Russell Greiner}{cs,amii}
\end{icmlauthorlist}

\icmlaffiliation{cs}{Computing Science, University of Alberta, Edmonton, Canada}
\icmlaffiliation{ece}{Eletrical Computer Engineeering, University of Alberta, Edmonton, Canada}
\icmlaffiliation{amii}{Alberta Machine Intelligence Institute, Edmonton, Canada}

\icmlcorrespondingauthor{Shi-ang Qi}{shiang@ualberta.ca}
\icmlcorrespondingauthor{Russell Greiner}{rgreiner@ualberta.ca}

\icmlkeywords{Machine Learning, ICML}

\vskip 0.3in
]



\printAffiliationsAndNotice{}  

\begin{abstract}
Discrimination and calibration represent two important properties of survival analysis, 
with the former assessing the model's ability to accurately rank subjects and the latter evaluating the alignment of predicted outcomes with actual events. 
With their distinct nature, it is hard for survival models to simultaneously optimize both of them especially as many previous results found improving calibration tends to diminish discrimination performance.
This paper introduces a novel approach utilizing \emph{conformal regression} that can improve a model's calibration without degrading discrimination. 
We provide theoretical guarantees for the above claim, and rigorously validate the efficiency of our approach across 11 real-world datasets, showcasing its practical applicability and robustness in diverse scenarios.
\end{abstract}

\section{Introduction}
\label{sec:intro}

Survival analysis is a useful tool, especially in biomedical applications. 
Given a patient's covariates, we focus on survival models that can
provide a time-to-event distribution (aka, survival distribution).
Accurate estimates of these times-to-event can help clinicians prescribe personalized treatment strategies or allocate resources for the healthcare system. 
Learning such survival models is challenging 
as many training datasets include 
\emph{censored} subjects -- 
instances where we only know a lower bound of their time-to-event.

Many survival models seek to optimize discriminative performance, which prioritizes the ability to accurately rank subjects~\cite{harrell1996multivariable} in terms of their survival probabilities or time-to-event times. 
Discrimination measurements are useful if, for example we want to rank the severity levels of the patients on the waiting list so that we can prioritize the most severe to receive the earliest surgery.

While discrimination is helpful, 
this evaluation 
might not assess some important aspects of a prediction.
For example,
imagine
a model estimates Mr.~Smith's chance of survival as 90\% within 1 year, 50\% in 2 years, and 30\% in 10 years.
Discrimination does not reflect how well the model's predicted survival probabilities match the actual distribution of observations across a population
-- a concept known as calibration~\cite{d2003evaluation, haider2020effective}.
There are two main types of calibration measurements: distribution calibration (D-cal) and single-time calibration (1-cal), which offer distinct but related insights (see Section~\ref{sec:background} for illustration and comparison). 
Calibrated predictions are essential as they ensure the reliability of survival probabilities, which is fundamental for both individual assessments and population-level decision-making strategies.
{By evaluating how well a model corresponds to the world,}
calibration aids in aligning model predictions closely with real-world outcomes.
The importance of calibration is emphasized across various fields --
\eg forecasting~\cite{degroot1983comparison}, Bayesian analysis~\cite{dawid1982well}, and risk assessment~\cite{pepe2013methods}. 

Many research projects attempt to enhance calibration performance by incorporating a calibration-related loss function into the optimization process.
However, this approach presents several challenges during the model's training phase, including increased difficulty in model convergence and the need to tune more hyperparameters.
More importantly, many of these approaches found that
improving calibration tends to diminish discriminative performance~\cite{goldstein2020x, chapfuwa2020calibration}. 
This phenomenon, known as the \emph{discrimination-calibration trade-off}~\cite{kamran2021estimating}, limits the applicability of these methods, particularly in contexts where discriminative accuracy remains important.

\paragraph{Contributions:} To address this challenge, this paper introduces the 
{\em Conformalized Survival Distribution (CSD)}\ framework --
a plug-in post-processing method designed to enhance the calibration of a survival distribution model, without compromising its discriminative power. 
This model-agnostic framework is compatible with all statistical and machine-learning survival models as long as they can predict individual survival curves (ISD). 
These learned models first discretize predicted survival curves into
percentile times (PCTs, defined in Section~\ref{sec:csd}), followed by adjusting these PCTs using conformal regression~\cite{romano2019conformalized}. 
The principal contributions of our study are:\\[-1.5em]
\begin{enumerate}
\parskip0em
\itemsep0em
    \item Developing the CSD framework, adapting conformal regression to censored data for making calibrated predictions.
    \item Demonstrating theoretically and empirically that CSD not only enhances calibration but also preserves discriminative performance.
    \item Empirically comparing CSD with other calibration-improving methods, showing that CSD is superior over many metrics.
    \item Connecting
    the two calibration measures for survival analysis by proving that minimizing D-cal is asymptotically equivalent to minimizing integrated 1-cal across all time points.
\end{enumerate}

From an application point of view, our findings suggest a shift in focus for researchers. 
Instead of concurrently optimizing discrimination and calibration during the training phase, 
they can now seek models with superior discriminative abilities and subsequently apply the CSD framework
to improve calibration. 
This approach simplifies the model development process while ensuring robust performance across 
these 
key metrics.

Section~\ref{sec:background} provides the background and introduces the definition of discrimination and calibration in survival analysis.
Section~\ref{sec:csd} describes the CSD method and provides guarantees that CSD can not only enhance calibration but also preserve discriminative performance.
Section~\ref{sec:exp} presents the extensive empirical analysis across 11 real-world survival datasets and shows the effectiveness of CSD.
A Python implementation of CSD is available online at \url{https://github.com/shi-ang/CSD}, along with code to replicate all experiments.

\begin{figure*}[ht]
    \centering
    \vspace{-0.1in}
    \includegraphics[width=\textwidth]{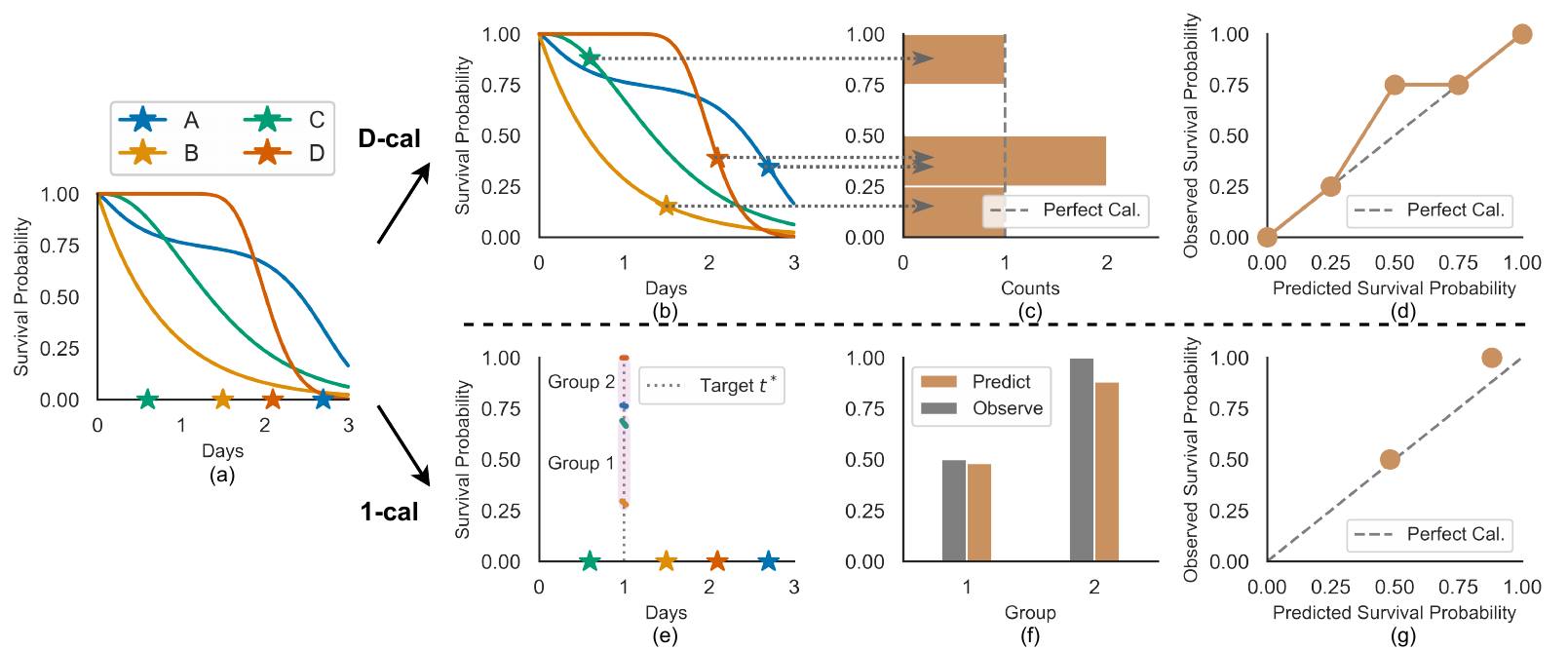}
    \vspace{-0.2in}
    \caption{Comparison of distribution calibration (upper part) and single-time calibration (lower part) using four uncensored patients. The leftmost (a) includes the predicted survival distributions (solid curves) and the true labels (stars) for four patients. \textbf{The upper section shows that D-cal process:} (b) obtains the predicted probabilities of each patient's event time; (c) constructs a histogram for these probabilities, where an ideally calibrated model would yield a uniformly shaped histogram; and (d) computes a probability–probability (P-P) plot assessing the histogram's uniformity. \textbf{The lower section (1-cal) process:} (e) obtains the predicted probabilities at a target time $t^*$ and groups the patients based on these sorted probabilities; (f) calculates average predicted and observed probabilities for each group, which should show statistical similarity; and (g) computes a P-P plot visualizing the similarity between predicted and observed probabilities.}
    \label{fig:calibration_illustration}
\end{figure*}

\section{Discrimination-Calibration Trade-Off}
\label{sec:background}
Survival prediction tools learn models that can predict survival for novel instances, from
a survival dataset $\mathcal{D}=\{(\Bfx{i}, t_i, \delta_i)\}_{i=1}^N$,
where each triplets corresponds to a subject, where $\Bfx{i}$ is the observed features for the $i$-th subject,
$t_i \in \mathbb{R}_+$ is the time of the event or censoring, 
and $\delta_i$ is a binary indicator where $\delta_i = 1$ means the subject had the event at time $t_i$ and $\delta_i = 0$ means the subject is right-censored (i.e., the subject has not experienced the event at time $t_i$). 
For each subject, there are two potential times of interest: the event time $e_i$ and the censoring time $c_i$. However, only the earlier of these two times is observable. 
Thus, we define $t_i \triangleq \min\{e_i, c_i\}$ and 
$\delta_i \triangleq  \mathbbm{1}[e_i \leq c_i]$, where $\mathbbm{1}[\cdot]$ is the indicator function. 
In this study, we make two assumptions:
\begin{itemize}
\parskip0em
\itemsep0em
    \item \emph{exchangeable}: the order of the subjects do not matter. For example, $\{(\Bfx{i}, t_i, \delta_i)\}_{i=1}^N$ are drawn i.i.d.
    \item \emph{conditional independent censoring}: event time is independent of the censoring time, given a description of the subject, \ie $e_i \ \bot \ c_i \mid \Bfx{i}$. 
\end{itemize}

Our goal is to estimate the individual survival distribution (ISD), $\Surv{t}{\Bfx{i}}=\prob(e_i > t \mid \Bfx{i})$, which represents the probability that the subject with feature $\Bfx{i}$ will survive beyond time $t$. 

\subsection{Discrimination}
\label{sec:cindex}
Discriminative performance measures how accurately a model ranks subjects based on predicted risk scores.\footnote{Note this is fundamentally differ from discrimination in other fields such as fairness~\cite{verma2018fairness}.}
It is often measured by the concordance index (C-index)~\cite{harrell1996multivariable},
which is the proportion of all comparable subject pairs whose predicted and outcome orders are concordant, defined as
\begin{equation}
\label{eq:c-index}
\begin{aligned}
    \text{C-index} (\{\eta_i\}_{i=1}^N, \Data)
    &= \frac{\sum_{i,j} \mathbbm{1}[t_i < t_j] \cdot \mathbbm{1}[\eta_i > \eta_j] \cdot \delta_i }{\sum_{i,j} \mathbbm{1}[t_i < t_j] \cdot \delta_i } ,
\end{aligned}
\end{equation}
where $\eta_i$ denotes the model's predicted risk score of $i$-th subject. The risk score can be defined as the negative of survival probability at a specified time or as the negative of predicted time-to-death. 
Many survival models aim to directly or partially optimize a C-index-related objective.
These include Cox proportional hazards (CoxPH)~\cite{cox1972regression}, as well as some recent approaches such as DeepHit~\cite{lee2018deephit} and Diffsurv~\cite{vauvelle2023differentiable}. 

\subsection{Calibration}
\label{sec:calibration}
\paragraph{Distribution Calibration} Calibration measures how well the prediction fits a set of observations.
Consider patient~$A$ whose predicted ISD (denoted $\SurvPred{t}{\Bfx{A}}$) is showned in Figure~\ref{fig:calibration_illustration}(a). The predicted median survival time for $A$ (the point at which the curve crosses 50\%) is around 2 days. 
We say this ISD is calibrated if, hypothetically, among 100 similar patients in parallel universes, 
about 50 experience the event before 2 days, and the remaining 50 after 2 days.
In reality, we cannot observe multiple outcomes for the same subject. 
Consequently, we shift our focus from subject-specific (conditional) calibration to a broader (marginal) calibration across the entire population. 
This leads to the concept of distribution calibration (D-cal)~\cite{haider2020effective}:
for any probability interval $[a, b] \subset [0, 1]$, let
\begin{equation}
\label{eq:d-cal}
    \hat{\Data}(a, b) = \{ [\Bfx{i}, t_i, \delta_i=1] \in \mathcal{D} \mid \SurvPred{t_i}{\Bfx{i}} \in [a, b]\}  ,
\end{equation}
be the subset of uncensored subjects in $\mathcal{D}$ whose predicted probability at their event time, $\Surv{t_i}{\Bfx{i}}$, falls within $[a, b]$. 
A model is D-calibrated if the proportion of patients $\frac{|\hat{\Data}(a, b)|}{|\Data|}$ is similar to the length of the interval $b - a$. 
The upper section of Figure~\ref{fig:calibration_illustration} demonstrates a D-cal assessment on 4 uncensored patients.
We first gather the set of predicted probabilities at the event times of patients $\{\SurvPred{t_i}{\Bfx{i}}\}_i$, and see if this set uniformly distributes across $[0, 1]$.
The calibration level can be visualized using the histogram (Figure~\ref{fig:calibration_illustration}(c)) and P-P plot (Figure~\ref{fig:calibration_illustration}(d)), or quantified via a $\chi^2$ test.
To integrate censored patients, we can ``split'' each censored patient uniformly to the subsequent probability intervals after the survival probability at censor time $\SurvPred{c_i}{\Bfx{i}}$~\cite{haider2020effective}, see~\eqref{eq:d-cal_censored_size} in Appendix~\ref{appendix:calibration_more}. 

\paragraph{Single-Time Calibration} Another form of calibration is the single-time calibration (1 cal)~\cite{d2003evaluation}, which evaluates the calibration of the predictions at a specific time. 
Considering the same patient A and its ISD, we can obtain the predicted survival probability at 1 day, \eg $\SurvPred{t^*=1}{\Bfx{A}} = 0.99$.
This prediction implies that for 100 similar patients, around 99 are expected to survive beyond 1 day. 
To formally calculate 1-cal at $t^*$, we first sort the predicted probabilities at $t^*$ for all patients and group them into $K$ groups, $G_1, \ldots, G_K$, as shown in Figure~\ref{fig:calibration_illustration}(e) (using $K=2$).
Each group $G_k$’s observed survival probability at $t^*$ is determined using the \citet{kaplan1958nonparametric} (KM) estimation, denoted as $S_{\KM (G_k)} (t^*)$ (not shown in the figure).\footnote{When there is no censorship, it is essentially the proportion of instances experiencing the event over time.}
We then compare the predicted survival probabilities, $\mathbb{E}_{\Bfx{i} \sim G_k} [\SurvPred{t^*}{\Bfx{i}}]$, with the KM-estimated observed probabilities.
The 1-cal assessment can be visualized using a histogram (Figure~\ref{fig:calibration_illustration}(f)), P-P plot (Figure~\ref{fig:calibration_illustration}(g)) or statistically evaluated via the Hosmer-Lemeshow test~\cite{hosmer1980goodness}.

\paragraph{KM Calibration} 1-cal, however, is limited to a single time point, whereas ISDs span all future times. 
To address this, \citet{chapfuwa2020calibration} proposed KM calibration (KM-cal), which compares the mean ISD curve for all subjects $\mathbb{E}_{\Bfx{i} \sim \Data} [\SurvPred{t}{\Bfx{i}}]$, against the overall KM survival curve.
The closer these two curves align, the more KM-calibrated the model is considered. 
This method essentially integrates 1-cal across all time points, using only one group ($K=1$) at each time. 
Appendix~\ref{appendix:calibration_more} provides further details and illustrations.


\paragraph{Remark} While calibration is clearly a desired property, it is important to recognize that focusing solely on calibration can be misleading,
\eg while the KM estimator is robustly calibrated, it cannot discriminate subjects as it provides identical predictions for all, meaning it does not discriminate any pairs of subjects.
In practice, the goal should be to identify a model that not only demonstrates good calibration but also possesses strong discriminative power. 

\begin{figure*}[t]
    \centering
    \vspace{-0.1in}
    \includegraphics[width=\textwidth]{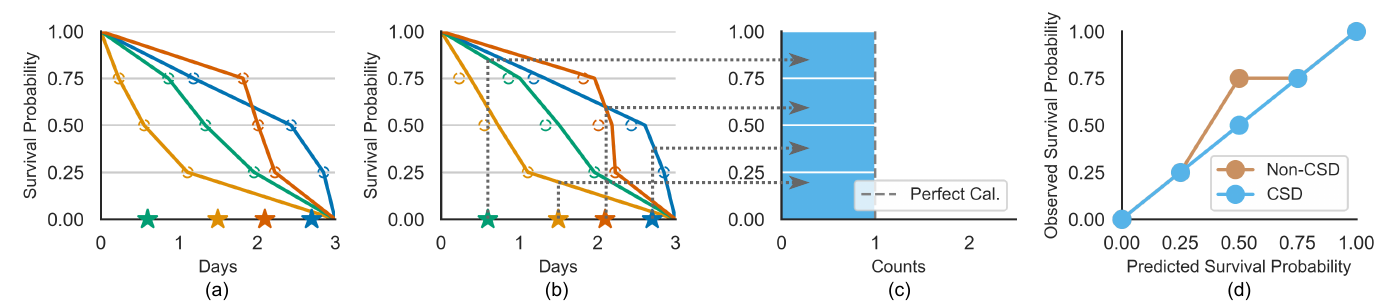}
    \vspace{-0.2in}
    \caption{Example of using Conformal Survival Distribution (CSD) to make the prediction D-calibrated, using the same patients and predictions shown in Figure~\ref{fig:calibration_illustration}. (a) Discretize the predicted survival distributions at three percentile levels (25\%, 50\%, 75\%); (b) Generate the new ISD by adjusting the PCTs, where the hollow points are the old PCTs; (c) Calculate the D-cal histogram using the adjusted ISDs; (d) P-P plot comparing the D-cal level between non-CSD and CSD predictions.}
    \label{fig:CSD-Illustration}
\end{figure*}

\subsection{Calibrated Survival Predictions}
\label{sec:bg_calibrated_survival_prediction}

Several approaches have been developed to enhance the calibration of ISDs,
usually by integrating a calibration-related loss function into the objective function.
For example, models optimized for the continuous rank probability score (CRPS)~\cite{avati2020countdown} can achieve good performance on 1-cal. 
\citet{goldstein2020x} introduces X-cal to the objective function, which is a differentiable approximation of D-cal. 
Similarly, \citet{chapfuwa2020calibration} incorporated the KM-cal term into the loss function, an approach called survival function matching (SFM).

However, discrimination and calibration serve distinct purposes --
\ie KM is well calibrated but has the lowest possible C-index, whereas models with perfect C-index may have poor calibration (see Appendix~\ref{appendix:disc_vs_cal} for detailed explanation).
When optimizing for multiple distinct objectives simultaneously, a model is prone to be dominated by one or several tasks, leading to potentially compromising performance in others~\cite{tang2020progressive}.
In fact, many~\cite{goldstein2020x, chapfuwa2020calibration, kamran2021estimating} found that improving the model's calibration performance will often harm the discrimination performance, sometimes significantly; see also our reproduced results in Figure~\ref{fig:short_results_objective-based_methods} and Section~\ref{sec:results_objective-based}.

In this paper, we introduce the CSD framework, which aims to reconcile this trade-off by disentangling calibration from discrimination in the optimization process.
In particular, this approach means the learner can focus on achieving optimal discrimination without considering the calibration.
Afterward, a post-processing CSD step will modify that model, to produce a new model with the same discrimination, but improved calibration.
Both theoretical and empirical evidence suggest that this post-processing approach retains the model's initial discriminative power while achieving calibration performance comparable to the KM estimator.

The conformalized survival analysis (CSA) proposed by~\citet{candes2023conformalized} bears a superficial resemblance to our method.
However, despite the similar name, our methods are fundamentally different.
First and most importantly, their CSA is applicable only to datasets with known censored times for all subjects, including those who have experienced an event (Type-I censoring) -- which is not available in many real-world situations. As right-censoring is a strict generalization of Type-I censoring, our CSD tackles a survival prediction task that is more general and realistic than CSA.
Secondly, CSA focuses solely on calibrating 90\% quantile predictions, which is less complex than our task of calibrating the entire ISD curve.
Thirdly, CSA's real-world experiments did not account for censoring in the calibration calculation.
Lastly, their study did not present the discrimination performance of the final model.\footnote{Due to their exclusive focus on calibration, they use only half of the dataset for training the model, leaving the other half to the conformalization; this will diminish the model's discrimination.}

\section{Conformalized Survival Distribution (CSD)}  %
\label{sec:csd}

We now describe how Conformal Survival Distribution (CSD) modifies an ISD prediction in a way that can improve the calibration while maintaining discrimination.
For simplicity, Section~\ref{sec:csd_overview} will initially ignore the censored subjects.
Section~\ref{sec:CSD_censor} will then integrate censored data into the CSD framework. 
Section~\ref{sec:data_split} will discuss practical considerations and extensions of CSD.

\subsection{Overview}
\label{sec:csd_overview}

Under the \emph{exchangeable} assumption, the CSD begins by splitting the dataset into two subsets: a training set $\Data_{\text{train}}$ and a conformal set $\Data_{\text{con}}$. 
Then, a given survival prediction algorithm will learn a model $\Model$ from $\Data_{\text{train}}$. 
Then we apply this model to $\Data_{\text{con}}$ to obtain initial predictions, \eg 
Figure~\ref{fig:calibration_illustration}(a) shows the ISDs predicted by the model for the four patients \{$A$, $B$, $C$, $D$\}.

In the next step, we discretize the ISD predictions at some specific percentile levels -- Figure~\ref{fig:CSD-Illustration}(a) shows 3 levels. For each patient $\Bfx{i}$, the corresponding times for these discretized percentile levels are defined as:
\begin{align*}
    \hat{q}_{\Model} (\rho \mid \Bfx{i}) \ &= \ \inf \{\, t: \hat{S}_{\Model}(t \mid \Bfx{i}) \leq \rho \, \} = \ \hat{S}^{-1}_{\Model}(\, \rho \mid \Bfx{i} \, ),
\end{align*}
where $\hat{S}^{-1}_{\Model}(\prob = \rho \mid \Bfx{i})$ is the inverse function of model's predicted ISD $\hat{S}_{\Model}(t \mid \Bfx{i})$, and $\rho$ is the $\rho$-th survival percentile.\footnote{Note that the $\rho$ in our paper is defined slightly different from the conformal regression~\cite{romano2019conformalized, candes2023conformalized}: 
our $\rho$ represents the survival percentile, while in their papers $\alpha$ is the quantile for the cumulative density function (CDF). That means our defined $\rho$ is complementary to theirs, $\rho = 1- \alpha$.} 
We refer to this $\hat{q}_{\Model} (\rho \mid \Bfx{i})$ as the predicted percentile time (PCT), which is represented by the hollow points in Figure~\ref{fig:CSD-Illustration}.
The motivation of the discretization is to facilitate the connection between the ISD predictions and D-cal calculation.
Specifically, PCT can form a prediction interval with the maximum time, \ie $[\hat{q}_{\Model} (\rho \mid \Bfx{i}), \infty]$. 
This interval indicates that there is $\rho$ probability that the actual event time for the subject will fall within this range.
Therefore, $\forall \ \rho \in [0, 1]$, if a model's predicted PCTs satisfy
\begin{align}
    \prob( \, t_i \in [\hat{q}_{\Model} (\rho \mid \Bfx{i}), \infty] \, ) \ = \ \frac{|\hat{\Data}(0, \rho)|}{|\Data|} \ \approx \ \rho  ,
\end{align}
we can say that this model is D-calibrated, where $\hat{\Data}(0, \rho)$ is calculated using~\eqref{eq:d-cal}.

Now we can adapt conformal regression~\cite{romano2019conformalized} after discretizing the ISDs. This approach allows us to compute conformity scores that quantify the error made by PCT (the horizontal distance between the stars and hollow points in Figure~\ref{fig:CSD-Illustration}(a)):  
\begin{equation}
\label{eq:conformal_score_uncensored}
    s_{j, \Model} (\rho)\ =\ \hat{q}_{\Model}(\rho \mid \Bfx{j})\, -\, t_j,
\end{equation}
and
\begin{equation*}
    \mathcal{S}_{\Model} (\rho)\ =\ \{\, s_{j, \Model} (\rho)\,\}_{j=1}^{|\Data_{\text{con}}|} ,
\end{equation*}
for each subject $j$ in the conformal set. 
To interpret the conformity score: 
high $s_{j, \Model}$ represents a poor agreement between the predicted interval and the true event time.
If $t_j$ is outside the predicted interval, then $s_j$ will be positive, which can be considered the magnitude of the error incurred by this mistake.
If $t_j$ correctly belongs to this interval, then $s_j$ will be nonpositive.

Finally, we can adjust the PCT using the following equation:
\begin{equation}
\label{eq:csd_adjust}
    \hat{q}_{\Model}' (\rho \mid \Bfx{i}) = \hat{q}_{\Model} (\rho \mid \Bfx{i}) - \Q \left[\, \rho; \mathcal{S}_{\Model}  (\rho) \, \right] ,
\end{equation}
where $\Q \left[{\rho; \mathcal{S}_{\Model} (\rho)}\right]$ denotes the $\frac{\left\lceil \rho (|\Data_{\text{con}}| + 1) \right\rceil}{|\Data_{\text{con}}|}$-th empirical quantile of $\mathcal{S}_{\Model} (\rho)$.\footnote{This is essentially the $\rho$-th quantile, with a small correction.} 
This calculation ensures that, for every percentile $\rho$, there is only $\rho \%$ of the subjects got included in their post-adjust prediction intervals $[\hat{q}_{\Model}' (\rho \mid \Bfx{i}), \infty]$. 
The post-adjust PCTs are shown in Figure~\ref{fig:CSD-Illustration}(b), where we can see that patient $D$ (red), whose probability at event time was originally located in $[0.25, 0.5]$, now moves to $[0.5, 0.75]$, making the histogram (c) and the P-P plot (d) aligned with the perfect calibration line. 
Note that all 4 ISDs were shifted, however, the other 3 patients stayed within their original percentile intervals.
Regarding the C-index, we can see that the \emph{relative order} of the predicted median survival times -- at the 50\% percentile intersection -- remains consistent before and after CSD.

We now prove that this generic CSD framework satisfied our proposed properties: (1) it does not diminish the discrimination performance; and (2) it improves the calibration performance (both D-cal and KM-cal). 

\begin{theorem}
\label{thm:c-index}
Applying the CSD adjustment to the percentile predictions does not affect the C-index of the model, regardless of whether to use the negative of the median or of the mean survival times as the predicted risk scores. 
\end{theorem}
\textit{Proof sketch.} For median survival times, this obviously holds because the term $\Q [\rho; \mathcal{S}_\Model (\rho)]$ can be considered as a constant (independent) wrt to $\Bfx{i}$. 
Therefore, the predicted median survival times $\{\hat{q}_{\Model}' (0.5 \mid \Bfx{i})\}_i$ will remain in the same order as $\{\hat{q}_{\Model} (0.5 \mid \Bfx{i})\}_i$. For mean survival times, Appendix~\ref{appendix:theorem3.1} provides the full proof. Note that this theorem is applicable to both uncensored and censored cases (in the next section).

\begin{theorem}
\label{thm:d-cal}
Under the exchangeable assumption, the percentile predictions constructed by CSD will exhibit exact distribution calibration, s.t., $\forall \, \rho \in [0, 1]$, we can have
\begin{equation}
    \rho \leq \prob(\,t_i \in [\hat{q}_{\Model}' (\rho \mid \Bfx{i}), \infty] \mid \Bfx{i}) \leq \rho + \frac{1}{|\Data_{\textnormal{con}}| + 1}.
\end{equation}
\end{theorem}
\textit{Proof sketch.} 
The theorem comes directly from the definition and is inspired by a standard proof in conformal regression~\cite{romano2019conformalized, angelopoulos2023conformal}.
For completeness, Appendix~\ref{appendix:theorem3.2} provides the full proof.

\begin{lemma}
\label{thm:KM-cal}
Under the exchangeable assumption, the percentile predictions constructed by CSD will asymptotically exhibit exact integrated calibration at all time points, which means that the prediction is KM-calibrated. 
\end{lemma}
\textit{Proof sketch.} We derive this by first building a connection between D-cal and KM-cal. We prove that D-cal and KM-cal are asymptotically equal, see Appendix~\ref{appendix:theorem3.3}. Then, under Theorem~\ref{thm:d-cal}, CSD will asymptotically be KM-calibrated.

\subsection{Handling Right-Censoring}
\label{sec:CSD_censor}
\subsubsection{Uncensored}
A naive approach to handle censoring is to discard the censoring subjects in the conformal set and compute the conformity scores as~\eqref{eq:conformal_score_uncensored}.
However, under the \emph{conditional independent censoring} assumption, it is easy to show that the data distribution with versus without censored subjects can be different, leading to non-calibration for the adjusted ISDs.

\subsubsection{Decensoring}
Another option to handle censored subject is to assign a ``best-guess'' (BG) value to each censored subject using non-parametric methods. 
After using those BG values as a surrogate for their true event times, we can use~\eqref{eq:conformal_score_uncensored} to calculate the conformity scores. Here, we consider two methods for calculating these surrogate values.

\emph{Margin}~\cite{haider2020effective} calculates the conditional expectation of the event time given a subject is censored at $c_i$ wrt the KM estimation, \ie $\E_t[S_{\KM}(t \mid t > c_i)]$. 

\emph{Pseudo-observation}~\cite{andersen2003generalised} (PO) uses a jackknife method to calculate the contribution of a censored subject to the group-level KM survival distribution.

Details of these two methods can be found in the Appendix~\ref{appendix:decensor}.

\subsubsection{KM-Sampling}
One problem with both of these decensoring methods is that they provide deterministic BG values as the surrogate times, which means the algorithm should be 100\% certain that these censored subjects have the event in that time, which means the algorithm only works when the BG values are all correct.

Therefore, instead of a single BG value, here we provide a \emph{BG distribution} for the event time. 
We first compute the KM estimation~\cite{kaplan1958nonparametric} for the training set, $S_{\KM}(t) = \prod_{i: t_i \leq t} (1 - \frac{d_i}{n_i})$, as shown by the tan curve in Figure~\ref{fig:KM-sampling}. Here $t_i$ is a time where at least one event happened, $d_i$ is the number of events that occur at $t_i$, and $n_i$ is the count of at-risk individuals (who have not yet had an event) before $t_i$. 
Then we can calculate the conditional group-level survival distribution, given a subject is censored at $c_j$ in the conformal set:
\begin{align}
    S_{\text{KM}}(\,t\, \mid\, t> c_j\,) \ &=\ \min
    \left\{
       \frac{S_{\text{KM}}(t)}{S_{\text{KM}}(c_j)},\ 1
       \right\}.
\end{align}

\begin{figure}[ht]
    \centering
    \vspace{-0.1in}
    \includegraphics[width=0.8\columnwidth]{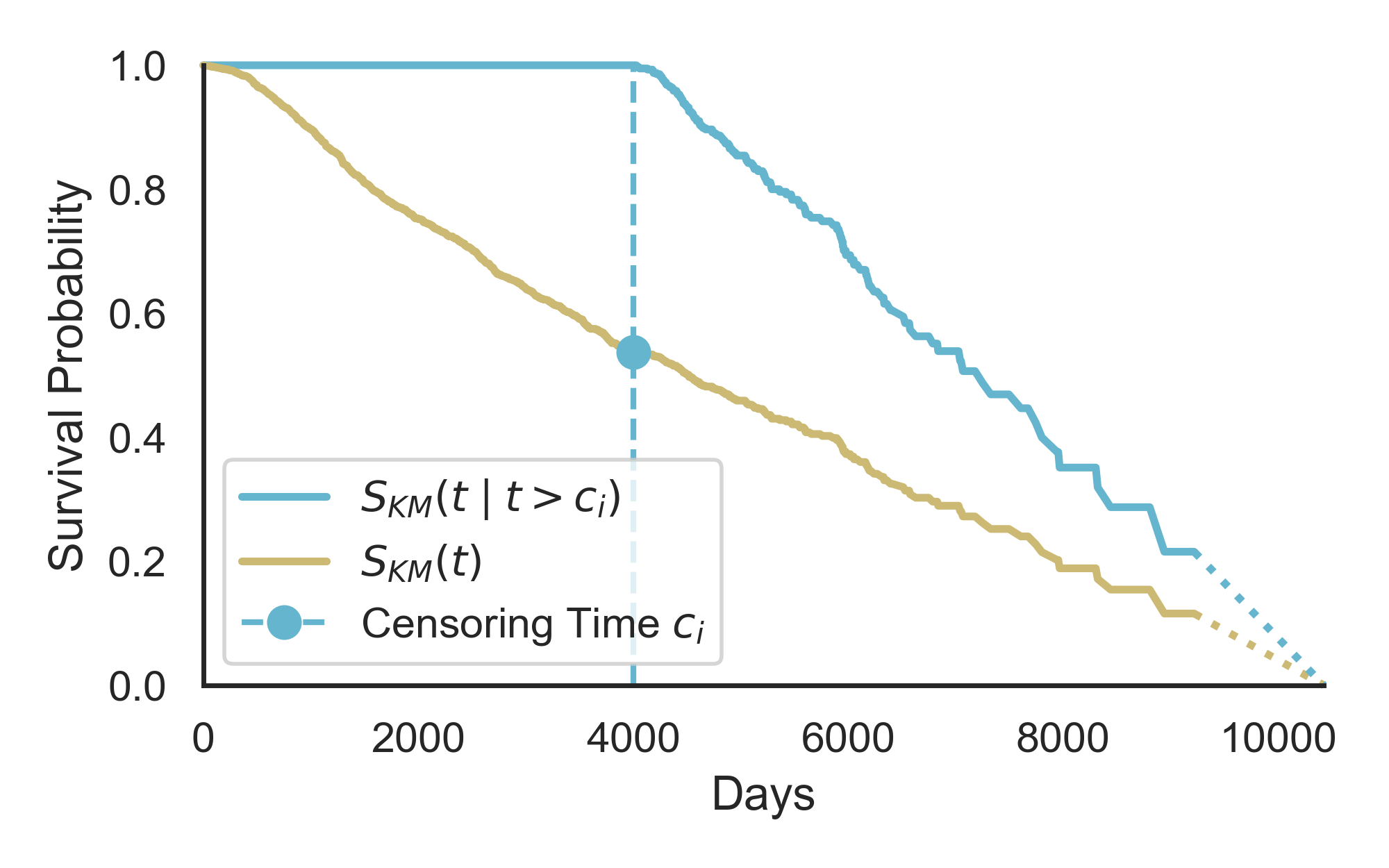}
    \vspace{-0.2in}
    \caption{An example of KM sampling from \texttt{METABRIC} dataset. }
    \label{fig:KM-sampling}
\end{figure}

The blue curve in Figure~\ref{fig:KM-sampling} is this conditional distribution, associated with $c_i = 4000$.\footnote{If the KM does not reach 0\% at the end of the curve, we adopt a linear extrapolation method~\cite{haider2020effective} to extend survival curves, using the starting point $[0, 1]$ and ending point $[t_{\text{max}}, S_{\text{KM}}(t_{\text{max}})]$, see the dotted extensions in Figure~\ref{fig:KM-sampling}.
}
To incorporate this BG distribution into the conformity score calculation, we can use the idea of sampling, \ie repeating a censored subject $R$ times, and for each, compute the conformity score by sampling an event time from the BG distribution.
Formally, for the $r$-th sampling, the conformity score is
\begin{equation}
\label{eq:conformal_score_censored}
    s_{j, \Model}^r (\rho) \ = \ \hat{q}_{\Model}(\rho \mid \Bfx{j}) - t_j^r,  \ \ \ t_j^r \sim S_{\text{KM}}(t \mid t > c_j).
\end{equation}
After each censored subject is sampled multiple times, the number of censored subjects in the conformal set will be significantly increased. 
Correspondingly, we need to balance the uncensored subjects.
Therefore, for each uncensored subject, we need to repeat it $R$ times as well, to maintain the same event-censoring ratio -- that is, we also sample $R$ times for each uncensored subject, using the Heaviside function at its event time.  
In this study, we use $R=1000$ for sampling.

The pseudo-algorithm for computing the CSD with the KM-sampling process is presented at Algorithm~\ref{alg:csd} in Appendix~\ref{sec:detailed_algorithm}.

\begin{remark}
This method depends on the KM model for calculating the BG distributions. 
Technically, we can use any survival estimator as the prior distribution for the subsequent sampling operation.
However, KM is guaranteed to be asymptotically calibrated (both D-cal and KM-cal, see Theorem~\ref{thm:km_is_d-cal} and~\ref{thm:km_is_km-cal}). In contrast, other survival algorithms that predict ISDs (\eg CoxPH) do not provide this guarantee, meaning they might perform very poorly here. Our empirical observations also align with this theoretical claim. That is, we could not find any ISD model that can consistently exhibit the same level of calibration performance as KM – see the calibration results in Section~\ref{sec:exp_results} and Figure~\ref{fig:full_results}.
\end{remark}

\begin{remark}

On the positive side, the KM curve does not make any assumption on the distribution. 
However, its applicability is limited to datasets with a sufficient number of events for accurate estimation.
Our ablation study (Section~\ref{sec:ablation}) reveals that for small datasets ($N< 1000$), the performance of the KM-sampling method is similar to both decensoring methods. 
CSD also has other constraints, \eg it cannot perform the conformalization step in batch mode due to $\Q$ operation, which means it is computationally inefficient.
We provide the computational analysis in Section~\ref{sec:exp_results} and Appendix~\ref{appendix:computational_analysis}.
\end{remark} 

\begin{remark}
CSD method makes two assumptions: \textit{exchangeability} and \textit{conditional independent censoring}. 

    \textit{Exchangeability} is a standard assumption in machine learning. For example, if the average life expectancy is around 70 for the first 100 instances (collected now), then we assume the average will still be around 70 for the next 100 instances, etc. If this assumption is violated, then any model trained on the current dataset may be ineffective and non-calibrated when it is tested on future instances. We have endeavored to uphold this assumption as rigorously as possible during our experiments.

    \textit{Conditional independent censoring} is the fundamental assumption in survival analysis. This assumption underpins the development of many survival models and the KM-sampling method. While this assumption may not always hold in practical scenarios, we intend to include a broad array of real-world datasets to mitigate the impact of any potential assumption violations.

\end{remark}

\subsection{Dataset Splitting}
\label{sec:data_split}
As mentioned earlier, CSD begins by splitting the data into a training and a conformal set. 
Previous methods~\cite{romano2019conformalized, candes2023conformalized} recommend allocating 50\% of the data for training and the rest 50\% for conformal, ensuring sufficient samples in the calibration set.

However, this policy can be problematic:
reducing the training set size by half may cause underfitting of the model, resulting in suboptimal discriminative ability.
Since the post-conformal step does not enhance the model's discrimination ability, this method essentially sacrifices some discrimination performance for improved calibration, which is not our intended outcome.

To address this, we explore two alternative policies: 
(1)~Using the validation set, typically required for hyperparameter tuning or early stopping, as the conformal set; 
(2)~Combining the validation and training sets to form a larger conformal set, where the training set would be a subset of the conformal set. For models that do not require a separate validation set, we propose reusing the training set as the conformal set.
Section~\ref{sec:exp_results} discusses the performance of these two policies.

\section{Experiments and Results}
\label{sec:exp}
\subsection{Datasets}
To evaluate the efficacy of CSD, we trained the non-CSD baselines and their post-CSD competitors in 11 datasets. 
Table~\ref{tab:data_comp} offers an overview of the dataset statistics; 
see Appendix~\ref{appendix:data_details} for details.
We deliberately select a wide range of datasets with small-to-large sample sizes, small-to-large censoring rates, and small-to-large feature-sample ratios.

\begin{table}[t]
\centering
\vspace{-0.1in}
\caption{Summary of the datasets. We categorize datasets into \emph{small}, \emph{medium}, and \emph{large} groups based on sample sizes, using thresholds of 1,000 and 10,000 samples. $^\dagger$Number of raw features, with (number of features after one-hot encoding) in parentheses. }
\label{tab:data_comp}
\setlength\tabcolsep{1.8pt}
\begin{tabular}{lrrrc}
\toprule
\rowcolor{LightCyan} 
Dataset  & \#Sample    & \%Censored & Max $t$  & \#Feature$^\dagger$ \\   \midrule
\texttt{VALCT}    & 137         &  6.57\%  & 999      &  6 (8)                   \\
\texttt{DLBCL}    & 240         & 42.50\%  & 21.80     &  7399                   \\
\texttt{PBC}      & 418         & 61.48\%  & 4,795     &  17                   \\
\texttt{GBM}      & 595         & 17.23\%  & 3,881      &  8 (10)                   \\
\hdashline
\texttt{METABRIC} & 1,981       & 55.17\%  & 9218     & 79                    \\
\texttt{GBSG}     & 2,232       & 43.23\%  & 87.36      & 7                    \\
\texttt{NACD}     & 2,396       & 36.44\%  & 84.30      & 48                    \\
\texttt{SUPPORT}  & 9,105       & 31.89\%  & 2,029      & 26 (31)                   \\
\hdashline
\texttt{SEER-brain} & 73,703    & 40.12\%  & 227        & 10         \\
\texttt{SEER-liver} & 82,841    & 37.57\%  & 227        & 14         \\
\texttt{SEER-stomach} & 100,360 & 43.40\%  & 227        & 14         \\
\bottomrule
\end{tabular}
\end{table}

We split each dataset into a training set (90\%) and a testing set (10\%) using a stratified splitting procedure that balances both the time $t_i$ and the censor indicator $\delta_i$. 
For algorithms that require a validation set to tune hyperparameters or to early stop, we partition another balanced 10\% validation set from the training set. 
We run each dataset using 10 different random seeds and compute the mean and 95\% confidence interval (CI) on all evaluation metrics across the 10 splits.

\begin{figure*}[ht]
    \centering
    \includegraphics[width=\textwidth]{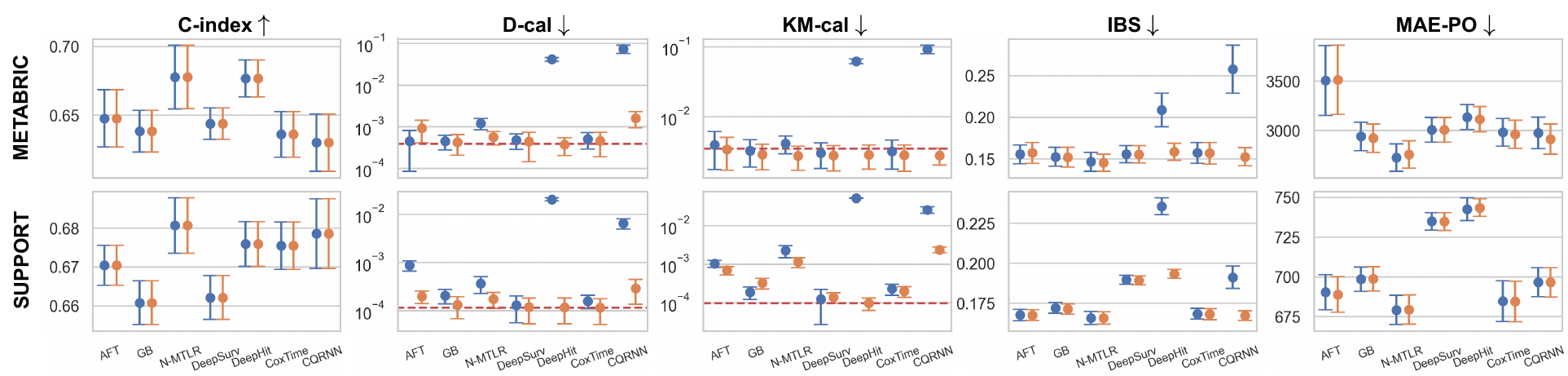}
    \vspace{-0.2in}
    \caption{Parts of the empirical results. The error bars represent mean and 95\% CI over ten runs, with blue denoting the non-CSD baseline and orange for the CSD-version. The red dashed lines represent the mean calibration performance for KM estimator, serving as an empirical lower-limit. 
    A higher C-index score indicates better performance, whereas lower scores are preferable for the other metrics.}
    \label{fig:main_results}
\end{figure*}

\subsection{Evaluation Metrics}
\label{sec:metrics}
Our goal is to enhance calibration performance while preserving the existing level of discriminative ability.
To test that, we use a discrimination metric (C-index) and two calibration metrics (D-cal and KM-cal) as the primary metrics.

For the C-index, we use the negative value of the predicted median survival time as the risk score in~\eqref{eq:c-index}. 
The D-cal metric is calculated by quantifying the mean squared distances between the observed survival probability and the predicted survival probability at each percentile (the distances between the dashed perfect calibration line and the model's solid P-P line in Figure~\ref{fig:calibration_illustration}(d)). 
Formally, for a set of percentiles $\mathcal{P} = \{0.1, 0.2, \ldots, 0.9\}$, D-cal is calculated by
\begin{equation*}
    \text{D-cal}\ =\ \frac{1}{|\mathcal{P}|} \sum_{\rho \in \mathcal{P}} \left( \frac{|\hat{\Data}_{\text{test}}(0, \rho)|}{|\Data_{\text{test}}|}\ -\ \rho \right)^2,
\end{equation*}
where $\hat{\Data}_{\text{test}}(0, \rho)$ is calculated using~\eqref{eq:d-cal}.
This score can also be interpreted as the area between the perfect calibration line (dashed line) and the model's P-P line (solid line) in Figure~\ref{fig:calibration_illustration}(d). The KM-cal is calculated using 
\begin{align*}
    \text{KM-cal} = \frac{1}{t_{\text{max}}} \int_{t=0}^{t_{\text{max}}} \left(S_{\text{KM} (\Data_{\text{test}})}(t) - \E_{\Bfx{i}} [\hat{S} (t \mid\Bfx{i}) ] \right)^2  dt,
\end{align*}
where $t_{\text{max}}$ is the maximum time in the testing set.
This score is the integrated discrepancy between the KM curve and the expected model prediction over time.

In addition to the primary metrics, we include two other metrics to assess survival prediction performance: the integrated Brier score (IBS)~\cite{graf1999assessment} measures the predicted probability accuracy across all times; and the mean absolute error with pseudo-observation (MAE-PO)~\cite{qi2023an} measures the accuracy of predicting the time to an event. 

\subsection{Models}
We compare the results using 7 survival algorithms: 
Accelerate Failure Time (AFT)~\cite{stute1993consistent} with Weibull distribution,
Gradient Boosting (GB)~\cite{hothorn2006survival} with component-wise least squares as base learner,
DeepSurv~\cite{katzman2018deepsurv}  (with Breslow extension), 
Neural Multi-Task Logistic Regression (N-MTLR)~\cite{fotso2018deep},
DeepHit~\cite{lee2018deephit}, 
CoxTime~\cite{kvamme2019time},
and censored quantile regression neural networks (CQRNN)~\cite{pearce2022censored}.
Appendix~\ref{appendix:model_details} describes the implementation details and hyperparameter settings.

Additionally, we also include a ``dummy'' model -- KM estimator~\cite{kaplan1958nonparametric} -- which uses the group-level KM for the training set as the same prediction to all test samples. 
Theoretically, it represents an extreme in our model spectrum, offering minimal discriminative performance (C-index$=0.5$) while theoretically achieving perfect calibration (Theorems~\ref{thm:km_is_km-cal} and~\ref{thm:km_is_d-cal}).
However, as evident in the results (Figure~\ref{fig:main_results} and Appendix~\ref{appendix:exp_complete}), this model still registers small scores for the D-cal and KM-cal. 
This is due to the slight distribution-shift caused by the dataset split, even though we have balanced the partition via both $\{t_i\}_{i=1}^N$ and $\{\delta_i\}_{i=1}^N$. 
Therefore, the calibration metric scores of the KM estimator can serve as the \textbf{empirical lower-limits} for the calibration performance.

\subsection{Results}
\label{sec:exp_results}

\begin{table}[!t]
\centering
\vspace{-0.1in}
\caption{Comparative Analysis of CSD performance by counting the number of times none-CSD baselines is better, CSD is better, and ties, over five evaluation metrics. The number in bracket means significantly better ($p<0.05$ using two-sided $t$-test). }
\label{tab:summary}
\vspace{1pt}
\setlength\tabcolsep{3.3pt}
\begin{tabular}{lccccc}
\toprule
         & C-index & D-cal & KM-cal & IBS & MAE-PO \\
\midrule
Non-CSD  & 3(0)    & 8(1)            & 20(7)    & 12(0)           & 30(0)           \\
CSD      & 13(0)    & \textbf{68(35)}            & \textbf{56(30)}    & \textbf{61(14)}           & \textbf{45(4)}           \\
ties     & \textbf{60}   &  0            & 0    & 3          & 1           \\
\bottomrule
\end{tabular}
\end{table}

We compare 7 baseline models and their post-CSD versions on 11 clinical datasets.
Note that for the \texttt{DLBCL} datasets, because the number of features (7399) is significantly greater than the sample size (240), AFT model fails to converge on this dataset.
Consequently, we have in total $11 \times 7 - 1 = 76$ comparisons.
Due to the space limit, 
the main text only reports the results on two datasets (\texttt{METABRAC} and \texttt{SURRPORT}) in Figure~\ref{fig:main_results}.
Appendix~\ref{appendix:exp_complete} presents the complete results.

\paragraph{Discriminative Performance}


The first column of Figure~\ref{fig:main_results} shows there are essentially no differences in the C-index between the non-CSD models and their CSD counterparts.
Specifically, out of 76 comparisons, 60 show equal performance (ties), and in 13 instances, the CSD models is (insignificantly) better (Table~\ref{tab:summary}).
This implies that the CSD framework preserves the discriminative performance in approximately 96\% of cases.
Moreover, in the remaining 3 cases where non-CSD models perform better, the differences are not statistically significant.

\paragraph{Calibrative Performance}

Our results, columns 2 and 3 of Figure~\ref{fig:main_results} and the summary in Table~\ref{tab:summary}, indicate a significant improvement in calibration performance due to the CSD framework.
Specifically, the CSD framework demonstrates improvement in 68 cases for D-cal (with 35 significantly better) and 56 cases for KM-cal (with 30 significantly better).
However, 
note that in some instances, baseline models may already reach optimal calibration levels, comparable to the KM estimator (\eg the AFT model in \texttt{METABRIC} dataset).
In such cases, further enhancement of calibration performance is challenging.  
Nevertheless, the CSD framework consistently performs well by maintaining the same calibration level in these already well-calibrated cases.

\paragraph{Other Metrics}

The fourth columns in Figure~\ref{fig:main_results} give the
IBS score, which is the expectation of Brier scores (BS) over time. 
According to~\citet{degroot1983comparison}, BS can be decomposed into a calibration part (1-cal) and a discrimination part (AUC). 
This implies that IBS, being an integrated form of BS, includes elements of KM-cal (integrated 1-cal).
Indeed, the results in Table~\ref{tab:summary} show that the trends in IBS closely mirror those in KM-cal, with CSD outperforming in 61 (and significantly in 14) of the 76 comparisons (80\%).

For MAE-PO, the performance enhancement by CSD is less pronounced compared to other metrics. 
Nonetheless, CSD still shows improvement in this metric, outperforming 45 (and significantly in 4) out of 76 comparisons (59\%).

\begin{figure}[t]
    \centering
    \includegraphics[width=\columnwidth]{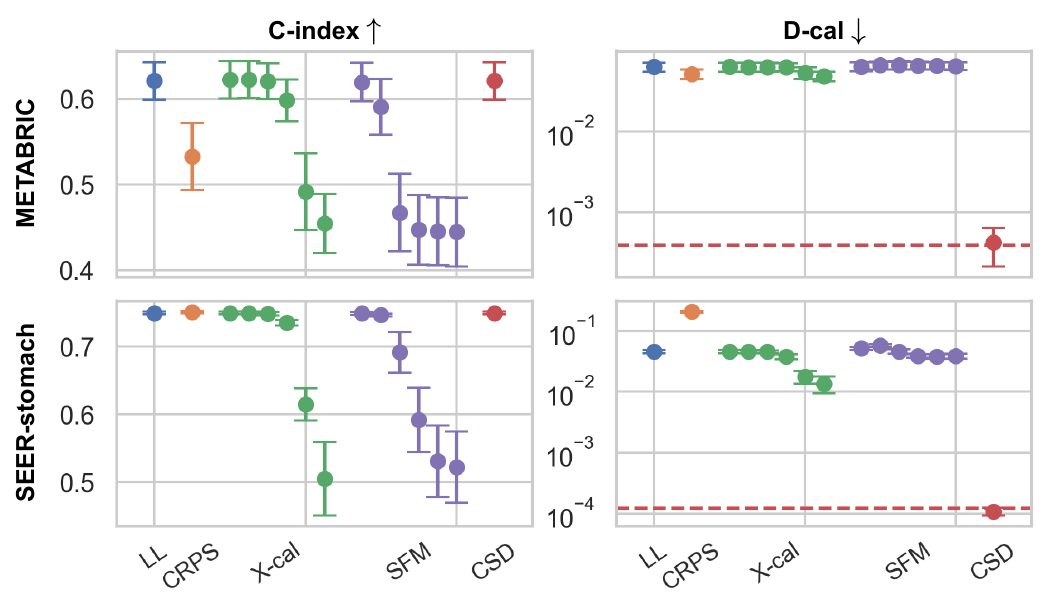}
    \vspace{-0.1in}
    \caption{Compare CSD with objective-based methods on a deep log-normal model. 
    The baseline (blue) uses likelihood loss (LL). 
    For X-cal and SFM methods, we gradually increase the weight for the calibration loss. The red dashed lines serve as the empirical lower-limit. }
    \label{fig:short_results_objective-based_methods}
\end{figure}

\paragraph{Compare with Objective-Based Methods}
\label{sec:results_objective-based}
In Figure~\ref{fig:short_results_objective-based_methods}, we compare CSD with three objective-based methods previously discussed in Section~\ref{sec:bg_calibrated_survival_prediction}: CRPS~\cite{avati2020countdown}, X-cal~\cite{goldstein2020x} and SFM~\cite{chapfuwa2020calibration}. 
The experimental procedures and the complete results for all 11 datasets are detailed in Appendix~\ref{appendix:objective-based_methods}.

Our results on 11 datasets indicate that X-cal can indeed enhance D-cal's performance by increasing the weight assigned to calibration loss. However, X-cal falls short in two ways: (1) it fails to achieve optimal calibration levels comparable to those achieved by KM and CSD (the red dashed line), and (2) it is often accompanied by a notable decrease in the C-index. Specifically, in \texttt{METABRIC} and \texttt{SEER-stomach}, as X-cal improves D-cal, the C-index drops to $\approx 0.5$.
Regarding CRPS, its beneficial effects on D-cal are less consistent, \eg it is effective for the \texttt{METABRIC} dataset but not for \texttt{SEER-stomach}. SFM also exhibits inconsistent trends; it is effective for \texttt{SEER-stomach} but not for \texttt{METABRIC}. 
In cases where CRPS and SFM are effective, they display trends comparable to those of X-cal.

In conclusion, the CSD framework significantly surpasses the performance of these objective-based methods.

\paragraph{Ablation Studies}
\label{sec:ablation}
CSD is influenced by three key hyperparameters: (1)~how to handle censorship; (2)~how to construct a conformal set; and (3)~the number of discretized percentiles. While Appendix~\ref{appendix:ablation} provides complete experimental details, results and summaries, the main findings are:\\[-2.em]
\begin{enumerate}
\parskip0em
\itemsep0em
    \item For medium to large datasets, the KM sampling method significantly outperforms other approaches.
    \item For large datasets, using the validation set alone as the conformal set is typically adequate. However, for smaller or medium datasets, it is beneficial to combine both validation and training sets to form the conformal set.
    \item Different numbers of percentiles have minimal impacts on performance. 
\end{enumerate}

\paragraph{Computational Complexity}

The quantile operation in~\eqref{eq:csd_adjust} requires storing the conformity scores for all data in the conformal set in advance, followed by computing the quantile amount in a substantial array. This introduces two types of complexity: space and time. In short summary, KM-sampling increases the space complexity by a factor of $R$ and the time complexity by approximately 16 times compared to the simple uncensored method. The complete results can be found in Appendix~\ref{appendix:computational_analysis}.

\section{Conclusion}
\label{sec:conclusion}

There are two crucial aspects in survival prediction:
\textbf{discrimination} assesses the model's ability to correctly rank subjects, and \textbf{calibration} evaluates the alignment of predicted distributions with actual outcomes.
Unfortunately, many survival prediction learners
encounter a trade-off between discrimination and calibration.

Our study introduces the conformalized survival distribution (CSD), a generic post-processing framework designed to mitigate this trade-off. 
Inspired by the principle of focusing on one task at a time, CSD disentangles calibration from discrimination during training. 
Specifically, we first produce a model, by concentrating solely on enhancing discrimination. 
Post-training, CSD steps in to refine the predicted survival distribution curves through conformal regression, producing new curves that are (in general) more calibrated.

Our empirical evidence shows that CSD successfully maintains discrimination ability on par with baseline models and enhances calibrative performance in terms of both D-cal and KM-cal. 
Also, CSD's effectiveness extends to two other widely used survival metrics, showing encouraging results.

\section*{Acknowledgements}
This research received support from the Natural Science and Engineering Research Council of Canada (NSERC), and the Alberta Machine Intelligence Institute (Amii). The authors extend their gratitude to the anonymous reviewers for their insightful feedback and valuable suggestions.

\section*{Impact Statement}
This paper presents work whose goal is to advance the field of Machine Learning. There are many potential societal consequences of our work, none of which we feel must be specifically highlighted here.




\bibliography{main}
\bibliographystyle{icml2024}

\newpage
\appendix
\onecolumn

\section{Further Discussion on Discrimination and Calibration Trade-Off}
\label{appendix:dis-cal_tradeoff}

\subsection{More about Discrimination}

This appendix delves deeper into the Concordance Index (C-index), which gauges the model’s ability to accurately rank subjects according to their risk levels. 
We previously introduced the concept and formal definition of the C-index in Section~\ref{sec:cindex}. 
To enhance comprehension, we now present a visual example to illustrate the calculation of the C-index.

\begin{figure}[h]
    \centering
    \includegraphics[width=0.9\textwidth]{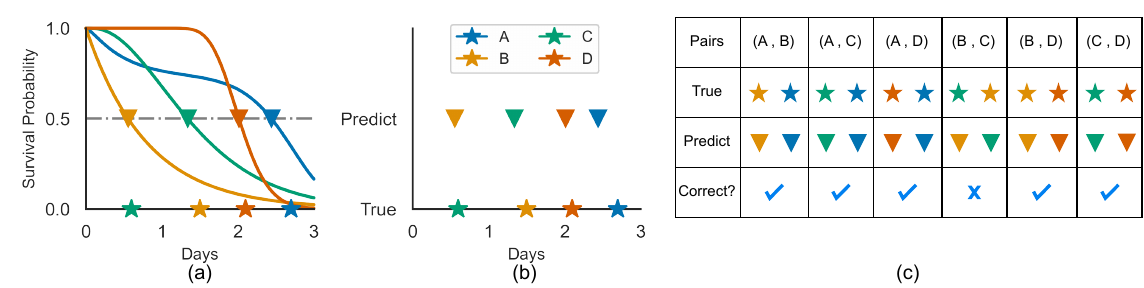}
    \caption{Illustration of concordance index (C-index), using the same patients as Figure~\ref{fig:calibration_illustration}.
    The stars denotes the true times, and the triangles denotes predicted median survival times.}
    \label{fig:cindex_example}
\end{figure}

Consider Figure~\ref{fig:cindex_example}, which depicts the true outcomes (indicated by stars) and the individual survival distribution (ISD) predictions (represented by curves) for four patients, based on the same example provided in Figure~\ref{fig:calibration_illustration}. 
Recall that predicted risk scores are typically negative of the predicted survival times (either median or mean). For simplicity and clearer visualization, we define risk scores as the negative median survival time, marked by the intersection of the curves with the 50\% dashed horizontal line (triangles), as illustrated in Figure~\ref{fig:cindex_example}(a).

Following the acquisition of the predicted median survival times, $\hat{t}_{i}^{(\text{median})}$, we compare their order with that of the actual event times. Consequently, we can reformulate the equation for the C-index as follows:
\begin{equation}
\begin{aligned}
\label{eq:cindex_reformat}
    \text{C-index} (\, \{\, \hat{t}_{i}^{(\text{median})} \,\}_{i=1}^N, \Data \, )
    \quad = \quad \frac{\sum_{i,j} \mathbbm{1}[\, t_i < t_j \,] \cdot \mathbbm{1}[\, \hat{t}_{i}^{(\text{median})} < \hat{t}_{j}^{(\text{median})} \,] \cdot \delta_i }{\sum_{i,j} \mathbbm{1}[\, t_i < t_j \,] \cdot \delta_i } \ .
\end{aligned}
\end{equation}
This equation can also be used for computing the C-index for the predicted mean survival times, simply by replacing the $\hat{t}_{i}^{(\text{median})}$ to $\hat{t}_{i}^{(\text{mean})}$.
In Figure~\ref{fig:cindex_example}, out of all $\binom{4}{2} = 6$ possible pairs formed from four patients, only the pair (B, C) is incorrectly predicted, as shown by the yellow and green triangles. Thus, the C-index score for this model prediction is $\frac{5}{6} \approx 83\%$.

To account for censored subjects, we do a modification by identifying \emph{comparable} pairs and examining the accuracy of the predicted order of those pairs. 
A pair is deemed comparable iff the ordering of their events can be unequivocally determined. 
Suppose that patient D (indicated in red) is censored, with the red star (at time 2.1) marking the censored time, while patients A, B, and C are fully observed. 
Among the pairs involving patient D, namely (A, D), (B, D) and (C, D), the pairs (B, D) and (C, D) remain comparable since it is evident that the event of D occurs after the events of B and C. 
However, the pair (A, D) becomes non-comparable because D's censored time precedes A's event time (at 2.5), leaving the true sequence of events unclear. 
Consequently, this pair is excluded from both the numerator and denominator in the C-index calculation as per equation~\eqref{eq:cindex_reformat}, resulting in a revised C-index score of $\frac{4}{5} = 80\%$.

\subsection{More about Calibration}
\label{appendix:calibration_more}

\paragraph{Distribution Calibration} 
D-cal examines the calibration ability across the entire range of predictions made by an individual survival distribution ISD model.
We can only trust a model's prediction if the ISD predictions are D-calibrated.
As we discussed in Section~\ref{sec:calibration}, for datasets without censorship, the D-calibration of the predictions can be evaluated by examining subsets of data within any probability interval $[a, b] \subset [0, 1]$. These subsets are defined as:
\begin{equation*}
    \hat{\Data}(a, b) \quad = \quad \{\, [\Bfx{i}, t_i, \delta_i=1] \in \mathcal{D} \mid \SurvPred{t_i}{\Bfx{i}} \in [a, b] \,\}  ,
\end{equation*}
where the proportion of observations within $\hat{\Data}(a, b)$ relative to the total dataset size $|\Data|$ should reflect the length of the interval, \ie $\frac{|\hat{\Data}(a, b)|}{|\Data|} = b - a$.
Formally, the size of such subsets is given by:
\begin{equation*}
    |\hat{\Data}(a, b)| \quad = \quad \sum_{i= 1}^N \delta_i \cdot  \mathbbm{1}\left[\, \SurvPred{t_i}{\Bfx{i}} \in [a, b] \,\right] .
\end{equation*}

To extend this analysis for censored observations, consider a censored patient with a predicted survival probability at the patient's censoring time $\SurvPred{c_i}{\Bfx{i}}$ (where $c_i = t_i$ if $\delta_i = 0$).
Given that censoring time provides a lower bound on the actual event time, it follows that $\SurvPred{c_i}{\Bfx{i}} > \SurvPred{e_i}{\Bfx{i}}$\footnote{This is because $c_i < e_i$ for $\delta_i = 0$, and the predicted ISD curve $\hat{S} (t\mid \Bfx{i})$ must be monotonically decreasing as the time $t$ increases.}.
Therefore, we can estimate the likelihood of the survival probability at the actual event time falling within the interval $[a, b]$, given we know the survival probability at the censor time, such as $\prob (\hat{S} (e_i \mid \Bfx{i}) \in [a, b] \mid \SurvPred{e_i}{\Bfx{i}} < \SurvPred{c_i}{\Bfx{i}} )$. 

For this purpose, we consider three scenarios:
\begin{itemize}
    \item \textbf{Case 1:} $a<b\leq\hat{S} (c_i \mid \Bfx{i})$
    \begin{align*}
        \prob \left(\hat{S} (e_i \mid \Bfx{i}) \in [a, b] \,\middle\vert\, \SurvPred{e_i}{\Bfx{i}} < \SurvPred{c_i}{\Bfx{i}} \right) \quad
        &= \quad \frac{\prob \left(a < \hat{S} (e_i \mid \Bfx{i}) < b, \ \SurvPred{e_i}{\Bfx{i}} < \SurvPred{c_i}{\Bfx{i}} \right)}{\prob \left(\SurvPred{e_i}{\Bfx{i}} < \SurvPred{c_i}{\Bfx{i}}\right)} \\
        &= \quad \frac{\prob \left(a < \hat{S} (e_i \mid \Bfx{i}) < b\right)}{\prob \left(\SurvPred{e_i}{\Bfx{i}} < \SurvPred{c_i}{\Bfx{i}}\right)} \\
        &= \quad \frac{b-a}{\SurvPred{c_i}{\Bfx{i}}};
    \end{align*}
    \item \textbf{Case 2:} $a\leq\hat{S} (c_i \mid \Bfx{i})<b$
    \begin{align*}
        \prob \left(\hat{S} (e_i \mid \Bfx{i}) \in [a, b] \,\middle\vert\, \SurvPred{e_i}{\Bfx{i}} < \SurvPred{c_i}{\Bfx{i}} \right) \quad
        &= \quad \frac{\prob \left(a < \hat{S} (e_i \mid \Bfx{i}) < \SurvPred{c_i}{\Bfx{i}}\right)}{\prob \left(\SurvPred{e_i}{\Bfx{i}} < \SurvPred{c_i}{\Bfx{i}}\right)} \\
        &= \quad \frac{\SurvPred{c_i}{\Bfx{i}}-a}{\SurvPred{c_i}{\Bfx{i}}};
    \end{align*}
    \item \textbf{Case 3:} $\hat{S} (c_i \mid \Bfx{i})<a<b$
    \begin{align*}
        \prob \left(\hat{S} (e_i \mid \Bfx{i}) \in [a, b] \,\middle\vert\, \SurvPred{e_i}{\Bfx{i}} < \SurvPred{c_i}{\Bfx{i}} \right) \quad
        &= \quad \frac{\prob (\varnothing)}{\prob \left(\SurvPred{e_i}{\Bfx{i}} < \SurvPred{c_i}{\Bfx{i}}\right)} \\
        &= \quad 0.
    \end{align*}
\end{itemize}
By integrating these scenarios, we can define the adjusted subset size $\hat{\Data}(a, b)$, accounting for both uncensored and censored patients, as follows:
\begin{equation}
\label{eq:d-cal_censored_size}
    |\hat{\Data}(a, b)| = \sum_{i= 1}^N \delta_i \cdot  \mathbbm{1}\left[\SurvPred{t_i}{\Bfx{i}} \in [a, b]\right] 
    + (1 -\delta_i) \left(\frac{\left(\SurvPred{t_i}{\Bfx{i}} - a\right)\mathbbm{1}\left[\SurvPred{t_i}{\Bfx{i}} \in [a, b]\right]}{\SurvPred{t_i}{\Bfx{i}}}
    + \frac{(b - a)\mathbbm{1}\left[\SurvPred{t_i}{\Bfx{i}} \geq b\right]}{\SurvPred{t_i}{\Bfx{i}}} \right).
\end{equation}

\paragraph{Single-time Calibration}

1-cal is a goodness of fit test to evaluate the calibration ability of risk predictions at a specific time.  
To perform this analysis at a chosen time point, denoted $t^*$, we begin by sorting the predicted probabilities for all patients at $t^*$ in order. These predictions are then divided into $K$ distinct groups: $G_1, \dots, G_k, \dots, G_K$.
Within each group, we calculate the expected number of events using the predicted probabilities, $\mathbb{E}_{\Bfx{i} \sim G_k} [\hat{S}(t^* \mid \Bfx{i})]$, and this is compared with the actual observed event rate, which is estimated using the KM method (because KM is an unbiased estimator for the population survival distribution). 
To assess the statistical similarity between the expected and observed rates of events, we apply the Hosmer-Lemeshow test~\cite{hosmer1980goodness}, based on its test statistic:
\begin{align}
\label{eq:hl_stats}
    \hat{\text{HL}} (t^*, \Data) \quad
    &= \quad \sum_{k=1}^{K} \frac{(O_k - n_k \Bar{p}_k)^2}{n_k \Bar{p}_k (1 - \Bar{p}_k)} \notag \\
    &= \quad \sum_{k=1}^{K} \frac{ 
    \left(|G_k|\cdot S_{\KM (G_k)}(t^*) - \sum_{\Bfx{i} \in G_k} \hat{S}(t^* \mid \Bfx{i})\right)^2}{\sum_{\Bfx{i} \in G_k} \hat{S}(t^* \mid \Bfx{i}) \left(1 - \frac{1}{|G_k|}\sum_{\Bfx{i} \in G_k} \hat{S}(t^* \mid \Bfx{i})\right)}.
\end{align}
where, $n_k$ is the total number of patients in the $k$-th group, $O_k$ is the observed number of patients who experienced the event before $t^*$ in the $k$-th group, and $\Bar{p}_k$ is the average predicted survival probability for the patients in the $k$-th group.

\paragraph{KM Calibration}
\citet{chapfuwa2020calibration} first proposed KM calibration as an alternation of 1-cal and D-cal.
The underlying principle of this approach is to ensure that the average predicted ISD curve for the test set is closely aligned with the KM curve for the testset.

Originally, KM calibration involved a straightforward visual comparison of the two curves. Later, \citet{yanagisawa23proper} refined this method by employing the Kullback-Leibler (KL) divergence to quantify the difference between the two distributions, as expressed in the following equation:
\begin{align*}
    d_{\text{KM-cal}} \quad
    &= \quad D_{KL} \left(S_{\text{KM} \Data_{\text{test}})}(t) \;\big\|\; \E_{\Bfx{i} \sim \Data_{\text{test}}} \left[\hat{S} (t \mid\Bfx{i}) \right] \right) 
\end{align*}
However, KL-divergence comes with notable drawbacks: (1) it does not satisfy the properties of metrics, such as symmetric and triangle inequality; (2) it can become undefined if the average predicted ISD, $\E_{\Bfx{i} \sim \Data_{\text{test}}} \left[\hat{S} (t \mid\Bfx{i}) \right]$, goes to zero for some $t$;
(3) it is sensitive to outliers if the average predicted ISD has a long tail at the end, etc.

To address these issues, our research calculates KM-cal as the integrated squared difference between the two curves:
\begin{equation*}
    \text{KM-cal} \quad = \quad \frac{1}{t_{\text{max}}} \int_{t=0}^{t_{\text{max}}} \left(S_{\text{KM} (\Data_{\text{test}})}(t) - \E_{\Bfx{i} \sim \Data_{\text{test}}} \left[\hat{S} (t \mid\Bfx{i}) \right] \right)^2  dt,
\end{equation*}
In practice, we first calculate the squared errors at all unique time points within the test set, followed by the application of the trapezoidal rule to compute the integral.

Moreover, this study establishes, for the first time, a link between KM-cal and earlier calibration metrics by demonstrating that KM calibration is equivalent to an unnormalized, integrated single-time calibration across all time points for a single group ($K=1$), as shown below:
\begin{proposition}
    KM calibration is equivalent to unnormalized integrated single-time calibration across all time points, for one group ($K=1$).
\end{proposition}
\begin{proof}
By calculating the Hosmer-Lemeshow test statistic for $K=1$, we find
\begin{align}
\label{eq:hl_stats_1group}
    \hat{\text{HL}} (t^*, \Data) \quad
    &= \quad \frac{ 
    \left(|\Data|\cdot S_{\KM (\Data)}(t^*) - \sum_{i= 1}^N \hat{S}(t^* \mid \Bfx{i})\right)^2}{\sum_{i= 1}^N \hat{S}(t^* \mid \Bfx{i}) \left(1 - \frac{1}{|\Data|}\sum_{i= 1}^N \hat{S}(t^* \mid \Bfx{i})\right)} \notag \\
    &= \quad \frac{ 
    \left(S_{\KM (\Data)}(t^*) - \E_{\Bfx{i} \sim \Data} \left[\hat{S} (t \mid\Bfx{i}) \right]\right)^2}{\frac{1}{|\Data|}\E_{\Bfx{i} \sim \Data} \left[\hat{S} (t \mid\Bfx{i}) \right] \left(1 - \E_{\Bfx{i} \sim \Data} \left[\hat{S} (t \mid\Bfx{i}) \right]\right)}. 
\end{align}
This equation illustrates that KM-cal integrates the numerator across all points, effectively representing an unnormalized statistic by omitting the denominator.
\end{proof}

\subsection{Discrimination versus Calibration}
\label{appendix:disc_vs_cal}

Discrimination and calibration are distinct concepts in model evaluation. 
A proof of this can be observed that the first and second derivative of~\eqref{eq:cindex_reformat} and~\eqref{eq:d-cal_censored_size} wrt to model's parameter do not match. 
However, the indicator functions in both equations make thing a bit complex.

\begin{figure}[ht]
    \centering
    \includegraphics[width=0.9\textwidth]{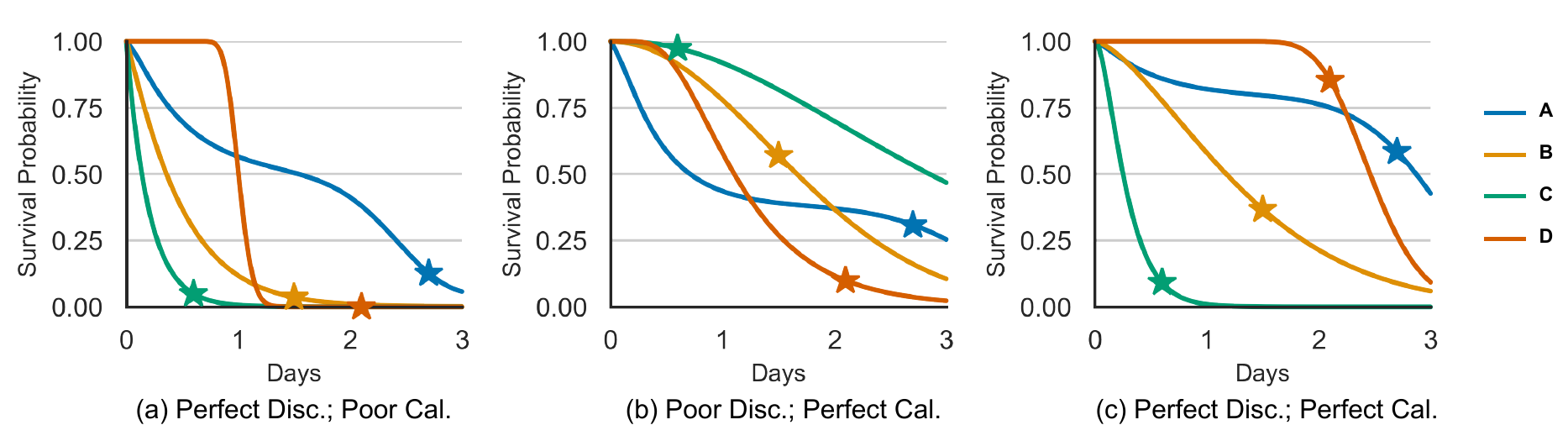}
    \caption{Illustration of discrimination and calibration using three hypothetical examples and the same patients in Figure~\ref{fig:calibration_illustration}.
    Here a perfect discrimination means the order of predicted median survival times (where curves intersect with 50\%) is aligned with true event times. 
    And a perfect (D-)calibration means the predicted probability at event times is uniformly distributed across $[0, 1]$ (one star in each probability interval).
    Disc.: discrimination; Cal.: calibration.
    }
    \label{fig:dis_vs_cal_example}
\end{figure}

In this section, we are using a much easier way to demonstrate their independence through counter-examples:
\begin{itemize}
    \item \textbf{A model $\Model_a$ can achieve perfect discrimination and poor calibration;}
    Figure~\ref{fig:dis_vs_cal_example}(a) shows the ISD prediction from $\Model_a$. 
    By checking the order in which the ISD curves intersect with 50\%, we can get the order of predicted median survival times as: green, yellow, red, blue (C, B, D, A), which is perfectly aligned with the order of stars (true event times), indicating perfect discrimination (C-index$=1$).
    However, as to the calibration, we can check the y-coordinate of the stars, which indicates the survival probability at event time. All the y-coordinates are located inside the interval $[0, 0.25]$, indicating the worst possible D-calibration statistics. 
    \item \textbf{A model $\Model_b$ can achieve poor discrimination and perfect calibration;}
    Figure~\ref{fig:dis_vs_cal_example}(b) shows the ISD prediction from $\Model_b$. 
    By checking the order in which the ISD curves intersect with 50\%, we can get the order of predicted median survival times as: blue, red, yellow, green (A, D, B, C), which is completely reversed from the order of stars (true event times), indicating lowest possible discrimination (C-index$=0$).
    However, as to the calibration, we can find that the y-coordinates of the stars are uniformly distributed among the y-axis (one in each interval), indicating the perfect D-calibration. 
    \item \textbf{A model $\Model_c$ can achieve perfect discrimination and perfect calibration;}
    Figure~\ref{fig:dis_vs_cal_example}(c) shows the ISD prediction from $\Model_c$. 
    By checking the order of the ISD curves intersecting with 50\%, we can get the order of predicted median survival times as: green, yellow, red, blue (C, B, D, A), which is perfectly aligned with the order of stars (true event times), indicating perfect discrimination (C-index$=1$).
    As to the calibration, we can find that the y-coordinates of the stars are uniformly distributed among the y-axis (one in each interval), indicating the perfect D-calibration. 
\end{itemize}

These examples clarify that a model can exhibit varying levels of discrimination and calibration, underscoring their independence as evaluation metrics.

\section{Why Kaplan-Meier models are perfectly calibrated}
\label{appendix:km_perfect_calibration}
In this section, we theoretically prove that the Kaplan-Meier (KM) model fitted on the training set is asymptotically calibrated (both KM-cal and D-cal) under some assumptions.
Therefore, the calibration scores for the KM model can be used as the \emph{empirical lower bound} for the performance comparison.

\begin{theorem}
\label{thm:km_is_km-cal}
    The Kaplan-Meier estimation on the training set is asymptotically KM-calibrated, under the \emph{exchangeable} and \emph{conditional independent censoring} assumptions.
\end{theorem}

\begin{proof}
According to the \emph{exchangable} assumption, the samples $\{(\Bfx{i}, t_i, \delta_i)\}_{i=1}^N$ are drawn i.i.d. from an arbitrary joint distribution.
Let's assume each subject corresponds to a ground truth condition survival distribution $S(t\mid \Bfx{i})$. 
Then under the law of large numbers, the ground truth marginal survival distribution for the training set should converge to the marginal survival distribution for the testing set, \ie
\begin{align}
    \E_{\Bfx{i} \in \Data_{\text{train}}} S(t\mid \Bfx{i}) \quad \rightarrow \quad \E_{\Bfx{j} \in \Data_{\text{test}}} S(t\mid \Bfx{j}).
\end{align}

Furthermore, KM estimator is proven to be unbiased for estimating the marginal survival distribution, under the assumption of \emph{conditional independent censoring}~\cite{hosmer2008applied}.
That means
\begin{align}
    S_{\KM(\Data_{\text{train}})}(t\mid \Bfx{i}) \quad \rightarrow \quad
    \E_{\Bfx{i} \in \Data_{\text{train}}} S(t\mid \Bfx{i}) \quad \rightarrow \quad \E_{\Bfx{j} \in \Data_{\text{test}}} S(t\mid \Bfx{j})
    \quad \rightarrow \quad
    S_{\KM(\Data_{\text{test}})}(t\mid \Bfx{i}),
\end{align}
where the final step is the ground truth for KM-cal.
This completes the proof.
\end{proof}

\begin{theorem}
\label{thm:km_is_d-cal}
    The Kaplan-Meier estimation on the training set is asymptotically D-calibrated, under the (1) \emph{exchangeable} assumption, (2) \emph{independent censoring} assumption, and (3) Kaplan-Meier estimation is \emph{strictly} monotonic decreasing.
\end{theorem}
\begin{proof}
Firstly, note that here we must use a different censoring assumption than in Theorem~\ref{thm:km_is_km-cal}.
For independent censoring assumption, we think the event and censoring times are independent, regardless of features, \ie $e_i \ \bot \ c_i$.
This is a stronger assumption, because $e_i \ \bot \ c_i$ can obviously lead to $e_i \ \bot \ c_i \mid \Bfx{i}$, but not the other way around.

\citet{haider2020effective} prove that the proportion $\frac{|\hat{\Data}(a, b)|}{|\Data|}$ is equal to the difference $b-a$ for the true ISD $S (t\mid \Bfx{i})$, under the assumption that the true ISD is strictly monotonic decreasing (see Theorem B.3 in~\citet{haider2020effective}).  
We can easily extend this proof to the KM estimation by applying the independent censoring assumption.
Specifically, independent censoring assumption is required for the probability decomposition step in~\citet{haider2020effective} proof (while the previous true ISD $S (t\mid \Bfx{i})$ only needs the less restrictive conditional independent censoring assumption).
Therefore, using these two assumptions, we can get that the KM model fitted on the testing set is D-calibrated.

To get to the last part (the KM model fitted to the training set is D-calibrated), we will need the \emph{exchangable} assumption, just like in Theorem~\ref{thm:km_is_km-cal}.
This completes the proof.
\end{proof}

\section{Proofs}
\subsection{Proof of Theorem 3.1}
\label{appendix:theorem3.1}

Before proving this theorem, we first recall some basic concepts in survival analysis.
Following the notations in Section~\ref{sec:background}, a trained survival model $\Model$ predicts the individual survival distribution (ISD) for subject $i$: $\hat{S}_{\Model} (t \mid \Bfx{i})$.
A predicted event time $\hat{t}_i$ can then be represented using this ISD by either median survival time: 
\begin{align}
    \hat{t}_{i}^{(\text{median})} \quad = \quad \inf \{t: \hat{S}_{\Model}(t \mid \Bfx{i}) \leq 0.5 \} \quad = \quad \hat{S}_{\Model}^{-1} (\rho = 0.5 \mid \Bfx{i}),
\end{align}
or using mean (expected) survival time:
\begin{align}
    \hat{t}_{i}^{(\text{mean})} \quad = \quad \E_t\left[\hat{S}_{\Model} (t \mid \Bfx{i})\right] \quad = \quad \int_0^{t_\text{max}} \hat{S}_{\Model} (t \mid \Bfx{i}) \, dt , \label{eq:mean_survival_time}
\end{align}
where we overuse the symbol $t_\text{max}$ to represent the time that ISD first reaches zero.

Recall Section~\ref{sec:background}, we need to discretize the ISD prediction for the conformal step using:
\begin{equation*}
    \hat{q}_{\Model} (\rho \mid \Bfx{i}) \quad = \quad \hat{S}_{\Model}^{-1} (\rho \mid \Bfx{i}),
\end{equation*}
then after the conformalization, we can recover the ISD prediction using the conformalized quantile prediction:
\begin{equation*}
    \hat{S}_{\Model}' (\rho \mid \Bfx{i}) \quad = \quad \hat{q}_{\Model}'^{-1} (\rho \mid \Bfx{i}).
\end{equation*}
Now we have everything we need for the proof. Here we formally restate Theorem~\ref{thm:c-index} in Section~\ref{sec:csd}.

\begin{theorem}
Applying the CSD adjustment to the percentile predictions does not impact the C-index of the model. Formally, $\forall \ i, j \in \Data$, for median survival times, we can have:
\begin{equation}
\label{eq:CSDprove_c-index_median}
    \text{given } \quad \hat{q}_{\Model} (\rho=0.5 \mid \Bfx{i}) > \hat{q}_{\Model} (\rho=0.5 \mid \Bfx{j}) \quad \Longrightarrow \quad \hat{q}_{\Model}' (\rho=0.5 \mid \Bfx{i}) > \hat{q}_{\Model}' (\rho=0.5 \mid \Bfx{j}),
\end{equation}
and for mean survival times, we have:
\begin{equation}
\label{eq:CSDprove_c-index_mean}
    \text{given } \quad  \int_{t = 0}^{t_{\text{max}}} \hat{S}_{\Model} (t \mid \Bfx{i})\ dt > \int_{t = 0}^{t_{\text{max}}} \hat{S}_{\Model} (t \mid \Bfx{j})\ dt \quad \Longrightarrow \quad \int_{t = 0}^{t_{\text{max}}} \hat{S}_{\Model}' (t \mid \Bfx{i}) \ dt > \int_{t = 0}^{t_{\text{max}}} \hat{S}_{\Model}' (t \mid \Bfx{j}) \ dt,
\end{equation}
with the assumption that $t_\text{max} < \infty$ is a finite number. 
\end{theorem}

\begin{proof}
For the first part of the proof (median survival time), let's take the partial derivative of~\eqref{eq:csd_adjust} with respect to $\Bfx{i}$ when $\rho = 0.5$:
\begin{align*}
    \frac{\partial \hat{q}_{\Model}' (\rho = 0.5 \mid \Bfx{i})}{\partial \Bfx{i}} \quad &= \quad 
    \frac{\partial \hat{q}_{\Model} (\rho = 0.5 \mid \Bfx{i})}{\partial \Bfx{i}} - \frac{ \partial \ \Q \left[\, {\rho= 0.5; \mathcal{S}_{\Model}} (\rho= 0.5) \,\right]}{\partial \Bfx{i}} \\
    &= \quad \frac{\partial \hat{q}_{\Model} (\rho = 0.5 \mid \Bfx{i})}{\partial \Bfx{i}}.
\end{align*}
Because the partial derivative of the second term is 0, we can see that the pre-CSD prediction $\hat{q}_{\Model} (0.5 \mid \Bfx{i})$ and post-CSD prediction $\hat{q}_{\Model}' (0.5 \mid \Bfx{i})$ have the same partial derivative. That means that the tangent lines at corresponding points should have the same slopes, or, put another way, their graph should go up and down in the same way when $\Bfx{i}$ changes. This proves~\eqref{eq:CSDprove_c-index_median}.

For the second part of the proof (mean survival time), we first use the predicted percentile times (PCT) to represent the mean survival time. Specifically, based on the integral of inverse function theorem, we can have:
\begin{align}
    \int_{t=0}^{t_{\text{max}}} \hat{S}_{\Model} (t \mid \Bfx{i})\ dt  + 
    \int_{\rho=\hat{S}_{\Model} (0 \mid \Bfx{i})}^{\hat{S}_{\Model} (t_{\text{max}} \mid \Bfx{i})} \hat{S}_{\Model}^{-1} (\rho \mid \Bfx{i})\ d\rho \quad &= \quad t_{\text{max}} \cdot \hat{S}_{\Model} (t_{\text{max}} \mid \Bfx{i}) - 0 \cdot \hat{S}_{\Model} (0 \mid \Bfx{i}) \notag \\
     \Longrightarrow \qquad\qquad\qquad \hat{t}_{i}^{(\text{mean})} + \int_{\rho = 1}^{\rho = 0} \hat{q}_{\Model} (\rho \mid \Bfx{i})\ d\rho \quad &= \quad t_{\text{max}} \cdot 0 - 0 \cdot 1 \notag\\
     \Longrightarrow \qquad\qquad\qquad\qquad\qquad\qquad\qquad\qquad \hat{t}_{i}^{(\text{mean})} \quad &= \quad \int_{\rho = 0}^{1} \hat{q}_{\Model} (\rho \mid \Bfx{i})\ d\rho \, . \label{eq:meantime_PCT}
\end{align}
Correspondingly, we can have the post-CSD mean survival time as
\begin{align}
    \hat{t}_{i}^{' \, (\text{mean})} \quad &= \quad \int_{\rho = 0}^{1} \hat{q}_{\Model}' (\rho \mid \Bfx{i})\ d\rho 
 \notag \\
    &= \quad \int_{\rho = 0}^{1}  \hat{q}_{\Model} (\rho \mid \Bfx{i}) - \Q [\, {\rho; \mathcal{S}_{\Model}} (\rho) \,] \ d\rho \, . \label{eq:meantime_PCT_post}
\end{align}
Therefore, by taking the partial derivative to both~\eqref{eq:meantime_PCT} and~\eqref{eq:meantime_PCT_post} with respect to $\Bfx{i}$ again (follow the same procedure above), 
we can easily find that pre- and post-CSD mean survival time have the same partial derivative.
Therefore, we can have~\eqref{eq:CSDprove_c-index_mean} proved.
\end{proof}

It's important to clarify a potential misunderstanding related to this theorem: there might be confusion that the CSD post-process does not alter the ordering of survival probability predictions. 
However, this is a misconception. 
The CSD modifies survival curves along the horizontal axis, specifically the time dimension. 
Consequently, this adjustment can lead to changes in the vertical sequence of survival probabilities. 
It is crucial to note that the theorem's claims are specifically concerning median and mean survival times, not the ordering of survival probabilities.

\paragraph{Remark} We now discuss about the reasonableness of the assumption that $t_\text{max} < \infty$ is a finite number.
We argue that it is a reasonable assumption in health-care related tasks, as the human lifespan is known to be a finite number (only cancer cells are immortal).

\subsection{Proof of Theorem 3.2}
\label{appendix:theorem3.2}

\begin{theorem}
Under the exchangeable assumption, the percentile predictions constructed by CSD will exhibit exact distribution calibration, s.t., $\forall \, \rho \in [0, 1]$, and $i \in \Data_{\text{test}}$, we can have
\begin{equation}
    \rho \quad \leq \quad \prob(t_i \in [\, \hat{q}_{\Model}' (\rho \mid \Bfx{i}), \infty \,] \mid \Bfx{i}) \quad \leq \quad \rho + \frac{1}{|\Data_{\textnormal{con}}| + 1}.
\end{equation}
\end{theorem}

\begin{proof}
This proof is a standard proof in conformal regression~\cite{romano2019conformalized, angelopoulos2023conformal}, with minor modifications to fit the context of survival analysis. Here for completeness, we will present the complete proof.

Let $s_{i, \Model} (\rho)$ be the conformity score~\eqref{eq:conformal_score_uncensored} for a test point $i$ and percentile level $\rho$. By the construction of the prediction interval $[\hat{q}_\Model' (\rho \mid \Bfx{i}), \infty]$, we have
\begin{align*}
    t_i \in [\, \hat{q}_\Model' (\rho \mid \Bfx{i}), \infty \,] \quad
    &\Longleftrightarrow  \quad t_i \ \geq \ \hat{q}_\Model' (\rho \mid \Bfx{i}) \ = \ \hat{q}_\Model (\rho \mid \Bfx{i}) - \Q[\, \rho; \mathcal{S}_\Model (\rho) = \{\, \hat{q}_\Model(\rho \mid \Bfx{i}) - t_j \,\}_{j=1}^{|\Data_{\text{con}}|} \,] \\
    &\Longleftrightarrow \quad  \Q[\, \rho; \mathcal{S}_\Model (\rho) \,] \ \geq \ \hat{q}_\Model (\rho \mid \Bfx{i}) - t_i \\
    &\Longleftrightarrow \quad \Q[\, \rho; \mathcal{S}_\Model (\rho) \, ] \ \geq \ s_{i, \Model} (\rho) ,
\end{align*}
therefore, by taking the expectation over the first and last term, we have
\begin{align}
\label{eq:conformal_lowerbound_half}
    \prob(t_i \in [\, \hat{q}_{\Model}' (\rho \mid \Bfx{i}), \infty \,] \mid (\Bfx{i}, t_i, \delta_i) \in \Data_{\text{test}}) \quad
    = \quad \prob (s_{i, \Model} (\rho) \ \leq \ \Q [\, \rho; \mathcal{S}_\Model (\rho) \,] \mid (\Bfx{i}, t_i, \delta_i) \in \Data_{\text{test}}).
\end{align}
Without loss of generality, we assume that the conformity scores for the conformal set are sorted in $\mathcal{S}_\Model (\rho)$, so that $s_1 < s_2 < \cdots < s_{|\Data_{\text{con}}|}$.
Therefore, for a new subject $i$ in $\Data_{\text{test}}$, its conformity score has the same probability of falling into any of the following intervals, $[-\infty, s_1], [s_1, s_2], \cdots, [s_{|\Data_{\text{con}}|}, \infty]$.
As we have $|\Data_{\text{con}}|$ conformity scores in the set, so we have $|\Data_{\text{con}}| + 1$ intervals. We conclude that for subject $i$ in the testset, we can have
\begin{equation*}
    \prob \left(s_{i, \Model} \ \leq \ \Q[\, \rho; \mathcal{S}_\Model (\rho) \, ] \right) \quad = \quad \frac{\lceil \rho (|\Data_{\text{con}}| + 1)\rceil}{|\Data_{\text{con}}| + 1} \quad = \quad \frac{\rho (|\Data_{\text{con}}| + 1) + (1-\epsilon)}{|\Data_{\text{con}}| + 1}.
\end{equation*}
The last equality is by converting the ceiling function to the original value plus a residual ($1-\epsilon$), where
$\epsilon$ is an infinitesimally small positive number,\footnote{This also requires that $0 \leq (1 - \epsilon) < 1$.} ensuring that the ceiling function just rounds up to the next integer.\footnote{This transformation is only achievable under the assumption of conformity score being continuous, avoiding ties. However, as pointed out by~\citet{angelopoulos2023conformal}, this condition can be solved by adding a small perturbation to the score.} By combining this equation with~\eqref{eq:conformal_lowerbound_half}, we have proved the theorem.
\end{proof}

\subsection{Proof of Lemma 3.3}
\label{appendix:theorem3.3}

\begin{lemma}
Under the exchangeable assumption, the percentile predictions constructed by CSD will asymptotically exhibit exact integrated calibration at all time points, which means that the prediction is KM-calibrated. 
\end{lemma}
\begin{proof}
We derive this by first building a connection between D-cal and KM-cal. We prove that D-cal and KM-cal are asymptotically equal. Then, under Theorem~\ref{thm:d-cal}, CSD will asymptotically be KM-calibrated.

\textbf{Let's first consider the uncensored cases, and prove the lemma in such circumstances.}

For the KM-cal with uncensored dataset, without the loss of generality, we can convert the integral from time 0 to maximum time to the summation over the discretized point $\{t_k\}_{k=1}^{k=K}$, where $t_K = t_{\text{max}}.$\footnote{Here we overuse the subscript $k$ to represent the index of the discretized time points, where in Section~\ref{sec:background} it represents the index for ordered group for 1-cal calculation.}
Then the KM-cal calculation is converted to: 
\begin{align}
    \mathcal{R}_{\text{KM-cal}}(\Data) \quad 
    &= \quad \frac{1}{t_{\text{max}}} \int_{t=0}^{t_{\text{max}}} \left(S_{\text{KM} (\Data)}(t) - \E_{\Bfx{i}} \left[\hat{S} (t \mid\Bfx{i}) \right] \right)^2  dt  \notag \\
    &= \quad \frac{1}{t_{K}} \sum_{k=1}^K \left( \frac{1}{N} \sum_{i=1}^N \mathbbm{1} [e_i > t_k] - \frac{1}{N} \sum_{i=1}^N \hat{S}(t=t_k \mid x_i)  \right)^2 \tag{No censorship} \\
    &= \quad \frac{1}{N^2 t_{K}} \sum_{k=1}^K \left(\sum_{i=1}^N \left( \mathbbm{1} [e_i > t_k] - \hat{S}(t=t_k \mid x_i)\right)\right)^2. \label{eq:integrated_1cal}
\end{align}

Same for D-cal, given a set of percentiles $\mathcal{P} = \{\rho_1, \rho_2, \ldots, \rho_L\}$ with $0 < \rho_1 < \cdots < \rho_l < \cdots < \rho_L \leq 1 $, D-cal is calculated by
\begin{align}
    \mathcal{R}_\text{D-cal} (\Data) \quad
    &= \quad\frac{1}{L} \sum_{l=1}^L \left( \frac{|\hat{\Data}(0, \rho_l)|}{N} - \rho_l \right)^2 \notag \\
    &= \quad\frac{1}{N^2L} \sum_{l=1}^L \left( |\hat{\Data}(0, \rho_l)| - N \cdot \rho_l \right)^2 \notag \\
    &= \quad\frac{1}{N^2L} \sum_{l=1}^L \left( \sum_{i=1}^N \left( \mathbbm{1}[\hat{S}(e_i \mid \Bfx{i}) < \rho_l]  - \rho_l \right)\right)^2. \label{eq:d-cal_reformat}
\end{align}

Now, if we compare \eqref{eq:d-cal_reformat} with \eqref{eq:integrated_1cal}, we can find that D-cal calculates over the percentile domain while KM-Cal does the calculation over the time domain. To make them talk using the same language, we first need to relate the percentiles with times. 

\begin{figure}[h]
    \centering
    \includegraphics[width=0.5\textwidth]{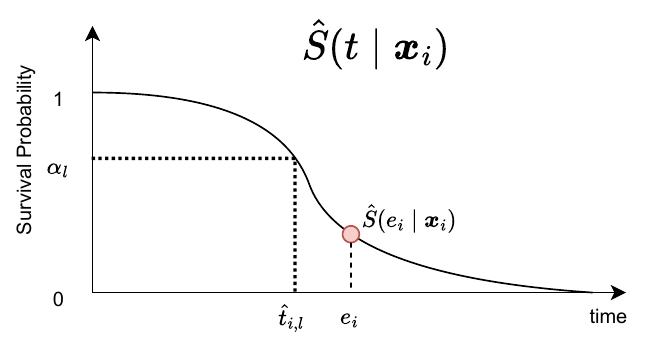}
    \caption{Relationship between $\hat{t}_{i, l}$, $\rho_l$, and $\hat{S} (t \mid \Bfx{i})$.}
    \label{fig:link_time_w_quantiles}
\end{figure}

Figure~\ref{fig:link_time_w_quantiles} shows the relationship between a percentile value and a time point. 
The curve represents the predicted ISD curve, $\hat{S}(t \mid \Bfx{i})$. 
$\rho_l = \hat{S} (t = \hat{t}_{i, l} \mid \Bfx{i})$ and its reverse, 
$\hat{t}_{i, l} = \hat{S}^{-1} (\rho= \rho_l \mid \Bfx{i})$. 
Based on this expression, we can have this proposition.
\begin{proposition}
\label{prop:transformation}
    For a strictly monotonic decreasing (no plateau area) ISD $\hat{S}(t \mid \Bfx{i})$,
    if the predicted probability at the event time, s.t. $\hat{S}(t=e_i \mid \Bfx{i})$ is less than $\rho_l$, then the event time $e_i$ must be greater than the corresponding percentile time $t_{i, l}$, and vice versa. Formally:
    \begin{equation*}
        \hat{S}(t=e_i \mid \Bfx{i}) < \rho_l  \quad \Longleftrightarrow \quad e_i > \hat{t}_{i, l} .
    \end{equation*}
\end{proposition}
Therefore, following this proposition, we can reformulate \eqref{eq:d-cal_reformat} to:
\begin{align}
\label{eq:d-cal_reformat2}
    \mathcal{R}_\text{D-cal} (\Data) 
    \quad =\quad \frac{1}{N^2L} \sum_{l=1}^{L}\left(\sum_{i=1}^N \left(\mathbbm{1} [\hat{t}_{i, l} < e_i] - \hat{S} (t=\hat{t}_{i, l} \mid \Bfx{i})  \right)\right)^2, 
\end{align}

Now, if we compare~\eqref{eq:integrated_1cal} with~\eqref{eq:d-cal_reformat2}, the only difference is that KM-Cal uses the discrete time indexed by $k$, and D-cal uses the time indexed by $i$ and $l$. 
However, if the two sets $\{t_k\}_{k=1}^{K}$  
and $\{\hat{t}_{i, l}\}_{i=1, l=1}^{N, L}$ are asymptotically equal ($K, N, L \rightarrow \infty$), then the integrated KM-cal and D-cal should also be asymptotically equal. 

This condition is not always satisfied, for population survival distribution models (\eg, Kaplan-Meier, Nelson-Aalen, or models that estimated one survival distribution for the entire population), the set $\{\hat{t}_{i, l}\}_{i=1, l=1}^{N, L}$ will just degenerate to $\{\hat{t}_{l}\}_{l=1}^{L}$, then we can easily align $\{\hat{t}_{l}\}_{l=1}^{L}$ with $\{t_k\}_{k=1}^{K}$.

However, for the individual survival distribution models, we cannot have $\{\hat{t}_{i, l}\}_{i=1, l=1}^{N, L}$ coincide with $\{t_k\}_{k=1}^{K}$ for almost all cases; this is mostly because the predicted curves can be a lot different across subjects, so we would end up with a very diverse $\hat{t}_{i, l}$. 
That being said, the more diverse the ISDs are, the greater the divergence between the D-cal and KM-cal.
This phenomenon is typical in statistical models, e.g., Cox proportional hazard (and its extensions, GB, DeepSurv) and AFT models are inherently (semi-)parametric, which here means that their ISD curves will have similar shapes. However, other models, \eg N-MTLR and CQRNN, produce ISD curves that can have very different shapes, therefore, larger divergence. 
But still, the above discussion provides a helpful prospective to connect the two important calibration methods for survival analysis.
\end{proof}

\paragraph{Remark}
This proof is not in conflict with the findings presented by~\citet{haider2020effective}, where they demonstrated that a model could achieve optimal D-cal yet exhibit subpar performance in terms of 1-Cal, and vice versa. 
However, it's important to note the distinction in the context of calibration being discussed: 
\citet{haider2020effective}'s argument pertains to 1-Cal at a specific point in time, whereas our proof addresses the concept of KM-cal, which considers 1-cal performance over a continuum of time points. 
This differentiation is akin to how a model might excel in terms of the Integrated Brier Score (IBS) across multiple time points but may still underperform when evaluated using the Brier Score (BS) at any given single time point.

\textbf{Let's now expand this proof to censored cases.}

\begin{proof}

For censored datasets, we rewrite the KM-cal as
\begin{align*}
    \mathcal{R}_{\text{KM-cal}}(\Data) \quad
    &= \quad\frac{1}{t_{\text{max}}} \int_{t=0}^{t_{\text{max}}} \left(S_{\text{KM} (\Data)}(t) - \E_{\Bfx{i}} \left[\hat{S} (t \mid\Bfx{i}) \right] \right)^2  dt  \notag \\
    &= \quad\frac{1}{N^2 t_{K}} \sum_{k=1}^K \left(\sum_{i=1}^N \left( S_{\text{KM} (\Data)}(t_k)- \hat{S}(t=t_k \mid x_i)\right)\right)^2,
\end{align*}
and D-cal as:
\begin{align*}
    \mathcal{R}_\text{D-cal} (\Data) \quad
    &= \quad \frac{1}{L} \sum_{l=1}^L \left( \frac{|\hat{\Data}(0, \rho_l)|}{N} - \rho_l \right)^2 \\
    &= \quad \frac{1}{N^2L} \sum_{l=1}^L \left( \sum_{i=1}^N \left( \frac{|\hat{\Data}(0, \rho_l)|}{N} - \rho_l \right)\right)^2,
\end{align*}
where the proportion of subset is
\begin{equation}
\label{eq:no_name}
    \frac{|\hat{\Data}(0, \rho_l)|}{N} \quad = \quad \frac{1}{N} \sum_{i= 1}^N \delta_i \cdot  \mathbbm{1}\left[\SurvPred{t_i}{\Bfx{i}} < \rho_l \right] 
    + (1 -\delta_i) \left(\mathbbm{1}\left[\SurvPred{t_i}{\Bfx{i}} <  \rho_l\right]
    + \frac{\rho_l \cdot \mathbbm{1}[\SurvPred{t_i}{\Bfx{i}} \geq \rho_l]}{\SurvPred{t_i}{\Bfx{i}}} \right).
\end{equation}
As previously demonstrated in the uncensored cases, the second terms of both equations, $\hat{S} (t = t_k \mid \Bfx{i})$ and $\rho_i$ converge asymptotically.
Now we only need to establish that the first terms, $S_{\text{KM} (\Data)} (t_k)$ and $\frac{|\hat{\Data}(0, \rho_l)|}{N}$, are also asymptotically equivalent.

Using the expression of~\eqref{eq:d-cal_censored_size}.
We can take the expectation of~\eqref{eq:no_name},
\begin{align*}
    \E_i \left[\frac{|\hat{\Data}(0, \rho_l)|}{N}\right ] \quad
    &= \quad \E_i \left[\mathbbm{1} \left[\SurvPred{e_i}{\Bfx{i}} < \rho_l \right] \cdot \mathbbm{1}\left[ e_i \leq c_i \right]\right]
    \ + \ \E_i \left[\mathbbm{1} \left[\SurvPred{c_i}{\Bfx{i}} < \rho_l \right] \cdot \mathbbm{1} \left[ c_i < e_i \right]\right] \\
    & \quad \quad +\E_i \left[ \frac{\rho_l}{\SurvPred{c_i}{\Bfx{i}}}  \mathbbm{1} \left[\SurvPred{c_i}{\Bfx{i}} \geq \rho_l \right] \cdot \mathbbm{1} \left[ c_i < e_i \right]   \right] \\
    &= \quad \prob \left(\SurvPred{e_i}{\Bfx{i}} < \rho_l \ , \ e_i \leq c_i \right) 
    \ + \ \prob \left(\SurvPred{c_i}{\Bfx{i}} < \rho_l \ , \ c_i < e_i \right) \\
    & \quad \quad + \E_i \left[ \frac{\rho_l}{\SurvPred{c_i}{\Bfx{i}}}  \mathbbm{1} \left[\SurvPred{c_i}{\Bfx{i}} \geq \rho_l \right] \cdot \mathbbm{1} \left[ c_i < e_i \right]   \right],
\end{align*}
using the transformation in Proposition~\ref{prop:transformation}, we can have
\begin{align}
\label{eq:exp_dcal}
    \E_i \left[\frac{|\hat{\Data}(0, \rho_l)|}{N}\right ] \quad
    &=  \quad \prob \left(e_i > \hat{t}_{i, l} \ , \ e_i \leq c_i \right) 
    \ + \ \prob \left(c_i > \hat{t}_{i, l} \ , \ c_i <e_i \right) \ + \ \E_i \left[ \frac{\rho_l}{\SurvPred{c_i}{\Bfx{i}}}  \mathbbm{1} [c_i \leq \hat{t}_{i,l} ] \cdot \mathbbm{1} \left[ c_i < e_i \right]   \right].
\end{align}

For the first two terms in the above derivation:
\begin{align*}
    \prob \left(e_i > \hat{t}_{i, l} \ , \ e_i \leq c_i \right) 
    + \prob \left(c_i > \hat{t}_{i, l} \ , \ c_i <e_i \right) \quad
    &= \quad \prob \left(e_i > \hat{t}_{i, l} \ , \ e_i \leq c_i \right) 
    + \prob \left(c_i> \hat{t}_{i, l}\ , \ e_i > \hat{t}_{i, l} \ , \ c_i <e_i \right) \\
    &= \quad \prob \left(e_i > \hat{t}_{i, l} \ , \ c_i > \hat{t}_{i, l} \right)\\
    &= \quad \prob(e_i > \hat{t}_{i, l}) \cdot \prob(c_i  > \hat{t}_{i, l}),
\end{align*}
where the probability decomposition is because of the conditional independent censoring assumption $e_i \ \bot \ c_i \mid \Bfx{i}$.
Then, for the third term:
\begin{align*}
    \E_i \left[ \frac{\rho_l}{\SurvPred{c_i}{\Bfx{i}}}  \mathbbm{1} [c_i \leq \hat{t}_{i,l} ] \cdot \mathbbm{1} \left[ c_i < e_i \right]   \right] \quad
    &= \quad \prob(\hat{S}(e_i \mid \Bfx{i}) \leq \rho_l)  \int_{c_i = 0}^{\hat{t}_{i, l}} \int_{e_i=c_i}^{\infty} \frac{1}{\SurvPred{c_i}{\Bfx{i}}} f(c_i, e_i \mid \Bfx{i}) \ de_i \ dc_i\\
    &= \quad \prob(e_i \geq \hat{t}_{i,l})  \int_{c_i = 0}^{\hat{t}_{i, l}} \int_{e_i=c_i}^{\infty} \frac{1}{\SurvPred{c_i}{\Bfx{i}}} f(c_i, e_i \mid \Bfx{i}) \ de_i \ dc_i\\
    &= \quad \prob(e_i \geq \hat{t}_{i,l})  \int_{c_i = 0}^{\hat{t}_{i, l}} \frac{1}{\SurvPred{c_i}{\Bfx{i}}} f(c_i \mid \Bfx{i}) \int_{e_i=c_i}^{\infty}  f(e_i \mid \Bfx{i}) \ de_i \ dc_i\\
    & = \quad \prob(e_i \geq \hat{t}_{i,l})  \int_{c_i = 0}^{\hat{t}_{i, l}} \frac{f(c_i \mid \Bfx{i})}{\SurvPred{c_i}{\Bfx{i}}} \SurvPred{c_i}{\Bfx{i}}  \ dc_i \\
    & = \quad \prob(e_i \geq \hat{t}_{i,l}) \int_{c_i = 0}^{\hat{t}_{i, l}} f(c_i \mid \Bfx{i}) \ dc_i \\
    &= \quad \prob(e_i \geq \hat{t}_{i,l}) \cdot \prob (c_i \leq \hat{t}_{i, l})
\end{align*}
Therefore, by combining the above two derivations to~\eqref{eq:exp_dcal}: 
\begin{align*}
    \E_i \left[ \frac{|\hat{\Data}(0, \rho_l)|}{N} \right] \quad = \quad \prob(e_i \geq \hat{t}_{i,l}) \cdot \prob(c_i  > \hat{t}_{i, l}) \ + \ \prob(e_i \geq \hat{t}_{i,l}) \cdot  \prob (c_i \leq \hat{t}_{i, l}) \quad = \quad \prob(e_i \geq \hat{t}_{i,l}).
\end{align*}
We know that the KM estimator is simply an empirical estimator of $\prob(e_i \geq t_k)$.
That implies that if the two sets $\{t_k\}_{k=1}^K$ and $\{\hat{t}_{i,l}\}_{i=1, l=1}^{N,L}$ become asymptotically equivalent as $K$, $N$, and $L$ approach infinity, then we can have $S_{\KM (\Data)} (t_k)= \E \left[\frac{|\hat{\Data}(0, \rho_l)|}{N}\right]$. 
Consequently, the KM-cal and D-cal are asympotically equal. This completes the proof.
\end{proof}
\section{Algorithm}

\subsection{Decensoring method}
\label{appendix:decensor}

\paragraph{Margin}
Margin~\cite{haider2020effective} assigns a ``best guess'' value (margin time) to each censored subject using the non-parametric population KM~\cite{kaplan1958nonparametric} estimator.
This margin time can be interpreted as a conditional expectation of 
the event time given the 
event time is greater than the censoring time~\cite{haider2020effective}. 
Given a subject censored at time $t_i$, we can calculate its margin time by:
\begin{equation*}
    e_{\text{margin}}(t_i, \Data) \quad = \quad  \mathbb{E}_t[e_i \mid e_i > t_i] \quad = \quad  t_i + \frac{\int_{t_i}^\infty S_{\KM(\Data)} (t) dt}{S_{\KM(\Data)} (t_i)} \ ,
\end{equation*}
where $S_{\KM(\Data)} (t)$ is typically derived from the training dataset. 

\paragraph{Pseudo-observations}
Another way to estimate surrogate event values is using pseudo-observations (PO)~\cite{andersen2003generalised, andersen2010pseudo}. 
Let $\{t_i\}_{i=i}^N$ be i.i.d. draws of a random variable time, $T$, and let $\hat{\theta}$ be an unbiased estimator for the event time based on right-censored observations of $T$. 
The PO for a censored subject is defined as: 
\begin{equation}
\label{eq:PO}
    e_{\text{PO}}(t_i, \Data) \quad = \quad N \times \hat{\theta} - (N-1) \times \hat{\theta}^{-i} \ ,
\end{equation}
where $\hat{\theta}^{-i}$ is the estimator applied to the $N-1$ element dataset formed by removing that $i$-th instance. 
The PO can be viewed as the contribution of subject $i$ to the unbiased event time estimation $\hat{\theta}$.
Here we can use, the mean survival time of the KM estimator, 
$\hat{\theta} = \E_t [ S_{\KM(\Data)} (t)] $ and $\hat{\theta}^{-i} = \E_t [ S_{\KM(\Data^{-i})} (t)] $ 
as unbiased estimators.

Some nice properties of PO values are:
(1)~PO values can be treated as though they are i.i.d. (Appendix D.1 in~\citet{qi2023an}). 
(2)~\citet{graw2009pseudo} has shown that, as $N \rightarrow \infty$, PO can approximate the correct conditional expectation:
\begin{equation*}
    \E [\,e_{\text{pseudo-obs}}(t_i) \mid \Bfx{i}\,]
    \quad\approx\quad \E [\,e_i \mid \Bfx{i}\,] \ ,
\end{equation*}
in situations where censoring does not depend on the covariates, and it also works well empirically in the situation where censoring is dependent on the covariates~\citet{binder2014pseudo}.

\subsection{CSD with KM-sampling}
\label{sec:detailed_algorithm}
In this section, we provide the pseudo-algorithm for the conformalized survival distribution with KM sampling.


\begin{algorithm}[ht]
   \caption{Conformalized Survival Distribution with KM sampling}
   \label{alg:csd}
\begin{algorithmic}[1]
   \STATE {\bfseries Input:} dataset $\Data$, testing set $\Data_{\text{test}}$,  Survival model $\Model$, Predefined percentile levels $\mathcal{P} = \{\rho_1, \rho_2, \ldots\}$, repeat parameter $R$
   \STATE {\bfseries Output:} ISD prediction for $\Data_{\text{test}}$
   \STATE Randomly split $\Data$ into a training set $\Data_{\text{train}}$ and a conformal set $\Data_{\text{train}}$
   \STATE Train the survival model $\Model$ using $\Data_{\text{train}}$ as the training set, and $\Data_{\text{con}}$ as the validation set.
   \STATE For every subject $j$ in $\Data_{\text{con}}$, make ISD predictions $\hat{S}_{\Model}(t \mid \Bfx{j})$
   \STATE Discretized the ISD predictions to PCTs using the predefined percentiles $\mathcal{P}$, using     $\hat{q}_{\Model} (\rho \mid \Bfx{i}) = \hat{S}^{-1}_{\Model}(\rho \mid \Bfx{j})$
   \STATE Calculate the KM estimation for the conformal set $S_{KM}(t)$
   \FOR{$r=1$ {\bfseries to} $R$}
   \FOR{$j \in \Data_{\text{con}}$}
   \IF{$\delta_j == 0$}
   \STATE Calculate the ``best-guess'' distribution $S_{\text{KM}} (t \mid t> c_j) = \min \left\{ \frac{S_{\text{KM}}(t)}{S_{\text{KM}}(c_j)}, 1\right\}$
   \STATE Sample $t_{j}^r = \sim S_{\text{KM}} (t \mid t> c_j)$
   \ELSE
   \STATE $t_{j}^r = t_j$
   \ENDIF
   \ENDFOR
   \ENDFOR
  \STATE Calculate the set of conformity scores $\mathcal{S}_{\Model, R}(\rho)$ for the conformal set using~\eqref{eq:conformal_score_censored}, for every $\rho \in \mathcal{P}$
   \STATE For every subject $i$ in $\Data_{\text{test}}$, make ISD predictions $\hat{S}_{\Model}(t \mid \Bfx{i})$
   \STATE Discretized the ISD predictions to PCTs using the predefined percentiles $\mathcal{P}$,      $\hat{q}_{\Model} (\rho \mid \Bfx{i}) = \hat{S}^{-1}_{\Model}(\rho \mid \Bfx{j})$
    \STATE Get the final prediction. $\hat{q}_{\Model}' (\rho \mid \Bfx{i}) = \hat{q}_{\Model} (\rho \mid \Bfx{i}) - \Q \left[{\rho; \mathcal{S}_{\Model, R}} (\rho)\right] $ 
\end{algorithmic}
\end{algorithm}

\subsection{Non-monotonicity after CSD}
Survival distribution is defined as the complement of CDF function, that means is defined to be always monotonically decreasing. 
Most survival models will apply a monotonic restriction to the output.\footnote{But not for all. For example, in the study, the CQRNN algorithm~\cite{pearce2022censored} does not have this monotonic constraint.} 
That means the discretized percentile prediction should still be monotonic, \ie $\hat{q}_{\Model}(\rho_1 \mid \Bfx{i}) \leq \hat{q}_{\Model}(\rho_2 \mid \Bfx{i})$ for arbitrary $\rho_1 > \rho_2$.
However, we could encounter some non-monotonic curves after the operation in~\eqref{eq:csd_adjust} in some rare cases. 
This usually occurs when $\rho$ is small (at the tail of the ISD curve).
This can be fixed by the bootstrap rearranging method~\cite{chernozhukov2010quantile}.

\paragraph{Remark} However, after fixing the non-monotonicity, there is no such guarantee that the predicted order remains the same, that means Theorem~\ref{thm:c-index} will not hold anymore. 
We anticipate that only few models and predictions will be affected by this, so the difference in C-index will be very slight. 
Surprisingly, as we show in Section~\ref{sec:exp_results} and Appendix~\ref{appendix:exp_complete}, this modification will actually make the C-index slightly better. 

\section{Experimental Details}
\label{appendix:exp_details}

\subsection{Datasets and Preprocessing}
\label{appendix:data_details}
In this section, we will describe how we preprocess the raw survival datasets.

The Veterans’ Administration Lung Cancer Trial (\texttt{VALCT}) dataset~\cite{kalbfleisch2011statistical} is derived from a randomized trial comparing two treatment regimens for lung cancer. It includes data from 137 patients, each described by 6 features. This dataset, recognized as a standard resource for survival analysis, is available in R's \texttt{survival} package.

Diffuse Large B-Cell Lymphoma (\texttt{DLBCL}) dataset~\cite{li2016multi} contains Lymphochip DNA microarrays from 240 biopsy samples of DLBCL tumors for studying the survival status of the corresponding patients and the observation lasts 21 years. It is a high-dimensional dataset, containing 7,399 features for each sample.

Mayo clinic Primary Biliary Cholangitis (\texttt{PBC}) data~\cite{therneau2000modeling} comprises data from 424 patients diagnosed with PBC, an autoimmune liver disease.
This is a standard survival analysis dataset available in the \texttt{survival} package in R.
In this data set, the status of each patient at the conclusion of the study is categorized into one of three types: censored, transplanted, or deceased. 
Our analysis aims to model the survival time following diagnosis. 
Therefore, we treat both censored and transplant cases as censored events ($\delta_i = 0$), and cases where the patient died as uncensored events ($\delta_i = 1$).
For handling missing data, we employ different strategies based on the type of feature. For continuous variables, we impute missing values using their mean. For categorical variables, we fill in missing values with the most common category (mode).

Glioblastoma multiforme (\texttt{GBM}) dataset is retrieved from The Cancer Genome Atlas (TCGA) dataset~\cite{weinstein2013cancer}.
We only select patients diagnosed with glioblastoma multiforme cancer to build the GBM data set. 
The TCGA data can be found in \url{http://firebrowse.org/} or by the instruction in~\citet{haider2020effective}.
The median values of radiation therapy, Karnofsky performance score, and ethnicity have been used to fill the missing values.

The Molecular Taxonomy of Breast Cancer International Consortium (\texttt{METABRIC}) dataset~\cite{curtis2012genomic} contains survival information for breast cancer patients.
This dataset includes a diverse range of feature sets that encompass clinical traits, expression profiles, copy number variation (CNV) profiles, and single nucleotide polymorphism (SNP) genotypes. All these features are derived from breast tumor samples collected during the METABRIC trial.
The dataset can be downloaded from (\url{https://www.cbioportal.org/study/summary?id=brca_metabric}), and it does not have any missing values.

Rotterdam \& German Breast Cancer Study Group (\texttt{GBSG}) dataset~\cite{katzman2018deepsurv} represents an amalgamation of data from two sources: the Rotterdam tumor bank~\cite{foekens2000urokinase} and the German Breast Cancer Study Group~\cite{schumacher1994randomized}.
The Rotterdam tumor bank dataset contains records for 1,546 patients with node-positive breast cancer
the GBSG contains complete data for 686 patients.
This combined dataset, which was already post-processed in the DeepSurv study~\cite{katzman2018deepsurv}, is available for download at \url{https://github.com/jaredleekatzman/DeepSurv/}.

The Northern Alberta Cancer Dataset (\texttt{NACD})~\cite{haider2020effective}
described 2402 patients with various cancers, including lung, colorectal, head and neck cancers, esophagus, stomach, and other cancers. 
The event of interest in this dataset is failure time. 
We drop patients with negative or zero survival time and leave with 2,396 patients.
There are no missing values for the features in this dataset.
The dataset can be downloaded from \url{http://pssp.srv.ualberta.ca} under ``Public Predictors''.

The Study to Understand Prognoses Preferences Outcomes and Risks of Treatment (\texttt{SUPPORT}) dataset~\cite{knaus1995support} comprises 9,105 participants with the aim of examining survival outcomes and clinical decision-making for seriously ill hospitalized patients. 
The dataset consists of a proportion of missing values for a large proportion of features. 
The official website (\url{https://biostat.app.vumc.org/wiki/Main/SupportDesc}) for the SUPPORT dataset provides a guideline for imputing baseline physiologic features, we followed that procedure. 
For the rest features with missing values, we will also use the mean value imputation for continuous features and the mode value imputation for categorical features.

Surveillance, Epidemiology, and End Results (\texttt{SEER}) Program dataset~\cite{seer} is a comprehensive collection of data on cancer patients in the United States. This dataset, which encompasses about 49\% of the U.S. population, includes vital information on patient diagnoses, survival times, and other relevant details sourced from various registries.
Our study focuses specifically on three distinct subsets extracted from the main \texttt{SEER} dataset: \texttt{SEER-brain}, \texttt{SEER-liver}, and \texttt{SEER-stomach}. These subsets, respectively, contain data on patients diagnosed with brain, liver, and stomach cancers. The primary objective of the datasets is to model the time elapsed from the diagnosis of these patients until a failure event, such as death or disease progression, following the data preprocessing steps by~\citet{farrokh2024effective}.
The feature set chosen for analysis includes age, sex, behavior recode, combined summary stage, grade, RX summary, Summary stage 2000, SEER historic stage A, derived AJCC TNM and stage group for the 6th edition, and derived AJCC stage group for the 7th edition. In preparing each subset for analysis, we first eliminated any features with more than 70\% missing or unknown values. We also excluded patients who had no recorded follow-up times, whether due to death or censoring. To ensure the uniqueness of the data, we removed any duplicate records, retaining only a single instance for each patient with matching clinical features and outcomes. In addition, patients with a recorded survival time of zero (indicating death or censoring on the same day as diagnosis) were also excluded from the study.
The SEER cohort is available for application and download from \url{www.seer.cancer.gov}.

\subsection{Evaluation Metrics}
All five metrics (C-index, D-cal, KM-cal, IBS, MAE-PO) are implemented in the \texttt{SurvivalEVAL}\footnote{\url{https://github.com/shi-ang/SurvivalEVAL}} package~\cite{qi2023survivaleval}.
The package can output the C-index, KM-cal, IBS, and MAE-PO scores directly. 
For D-cal, it generates the histogram and a $p$-value, we then need to calculate the D-cal statistics from the histogram, as discussed in Section~\ref{sec:exp}.

\subsection{Model Implementation Details}
\label{appendix:model_details}

In this section, we will describe the implementation of the baseline models utilized in the performance comparison.

\emph{AFT}~\cite{stute1993consistent} with Weibull distribution is a linear parametric model with two estimated coefficients (scale and shape), using a small L2 penalty on parameters during optimization. The model is implemented in \texttt{lifelines} packages~\cite{lifelines}.

\emph{GB}~\cite{hothorn2006survival} is an ensemble method with component-wise least squares as the base learner. We use the 100 boosting stages with partial likelihood loss for optimization
and 100\% subsampling for fitting each base learner. The model is implemented in \texttt{scikit-survival} packages~\cite{sksurv}.

\emph{N-MTLR}~\cite{fotso2018deep} is a neural network extension of the multi-task logistic regression model (MTLR)~\cite{yu2011learning}. 
MTLR is the first survival model that estimated the probability of survival of patients at each of a vector of discretized time points $[t_1, t_2, \ldots, t_{\text{max}}]$. 
To achieve that, MTLR set up a series of logistic regression models, for each time point.
Here we implement the extended N-MTLR proposed by~\citet{fotso2018deep}, which essentially attached a multilayer perceptron (MLP) before the linear MTLR layer. 
The number of discrete times is determined by the square root of numbers of uncensored patients, and use quantiles to divide those uncensored instances evenly into each time interval, as suggested in~\cite{jin2015using, haider2020effective}.
We reimplement the method based on the code provided in~\citet{kazmierski2020torchmtlr}, please see the code in the github repository.

\emph{DeepSurv}~\cite{katzman2018deepsurv} is a neural network adaptation of the Cox proportional hazard model (CoxPH)~\cite{cox1972regression}. 
CoxPH is a semi-parametric model with a proportional hazard assumption. It consists of a population-level baseline hazard function (baseline hazard, which is non-parametric) and a partial hazard function (relative hazard, which is parametric).
The original CoxPH model only predicts a relative hazard score (risk score) for each patient using the partial hazard function.
To make ISD prediction, we use the Breslow method~\cite{breslow1975analysis} to estimate the population-level baseline hazard function.
We add the ISD prediction to the method based on the code provided in~\citet{katzman2018deepsurv}, please see the code in the repository.

\emph{DeepHit}~\cite{lee2018deephit} is also a discrete model, like N-MTLR. 
It models the probability density function (PDF) of the event for each individual (and PDF can be used to calculate the survival distribution accordingly), while N-MTLR models the survival distribution directly. 
Furthermore, apart from the standard likelihood loss, it also contains a ranking loss term which changes the undifferentiated indicator function in~\eqref{eq:c-index} with an exponential decay function. 
Due to this reason, DeepHit might exhibit stronger discrimination power than other models (on the other hand, lower calibration power than others). 
The number and locations of discrete times are determined the same way as N-MTLR model (square root of numbers of uncensored patients, and quantiles).
The model is implemented in \texttt{pycox} packages~\cite{kvamme2019time}.

\emph{CoxTime}~\cite{kvamme2019time} is a non-proportional neural network extension of the CoxPH.
Besides the MLP layer that DeepSurv added,
CoxTime relieves the proportional assumption by allowing the baseline hazard function to model interactions between time $t$ and features $\Bfx{i}$ (before it was only time).
The model is implemented in \texttt{pycox} packages~\cite{kvamme2019time}.

\emph{CQRNN}~\cite{pearce2022censored} is a quantile regression-based method, which means instead of predicting survival probability given a certain time, it predicts survival time given a certain quantile.  
It uses the Portnoy's pinball loss and an expectation-maximization step to estimate the model.
Because the quantile regression-based methods do not enforce the prediction to be monotonic, we add the bootstrap-rearranging post processing~\cite{chernozhukov2010quantile} to the model's prediction.  
We reimplement the method based on the code provided in~\citet{pearce2022censored}, please see the code in the repository.

\paragraph{``Dummy model''}
\emph{Kaplan Meier (KM)~\cite{kaplan1958nonparametric}} is a non-parametric estimator to predict the survival distribution for a group of subjects. It is not a personalized prediction tool. 
As we discussed in Section~\ref{sec:background}, the KM model will theoretically achieve the perfect calibration.
Therefore, we use this model as the \textbf{empirical lower-limits} for the calibration performance comparison. 
See also Appendix~\ref{appendix:km_perfect_calibration} for why the KM model theoretically achieves the perfect calibration.
The model is implemented in \texttt{lifelines} packages~\cite{lifelines}.

In our study, we employ a uniform Multilayer Perceptron (MLP) architecture for a range of deep learning models, including \emph{N-MTLR}, \emph{DeepSurv}, \emph{DeepHit}, \emph{CoxTime}, \emph{CQRNN}, and \emph{LogNormalNN}\footnote{The introduction of the LogNormalNN model will be provided in subsequent Appendix~\ref{appendix:objective-based_methods}.}. 
Specifically, we configure the MLP with two layers, each containing 64 neurons, and build the model with batch normalization layers and rectified linear unit (ReLU) activation layers. 

For the training process, we utilize Adam optimizer combined with an L2 penalty for weight decay to fine-tune the models. The learning parameters are set as follows: a learning rate of 0.001, a batch size of 256, and a dropout rate of 0.4. Additionally, we implement an early stopping mechanism across all deep learning models, which is based on performance validation using a separate validation dataset. This approach ensures that the effectiveness of the models is not compromised by variations in the MLP configuration or other optimization parameters.

For further details, please refer to the code in the repository.

\section{Complete Experimental Results}
 \label{appendix:exp_complete}

\begin{figure}
    \centering
    \includegraphics[width=\textwidth]{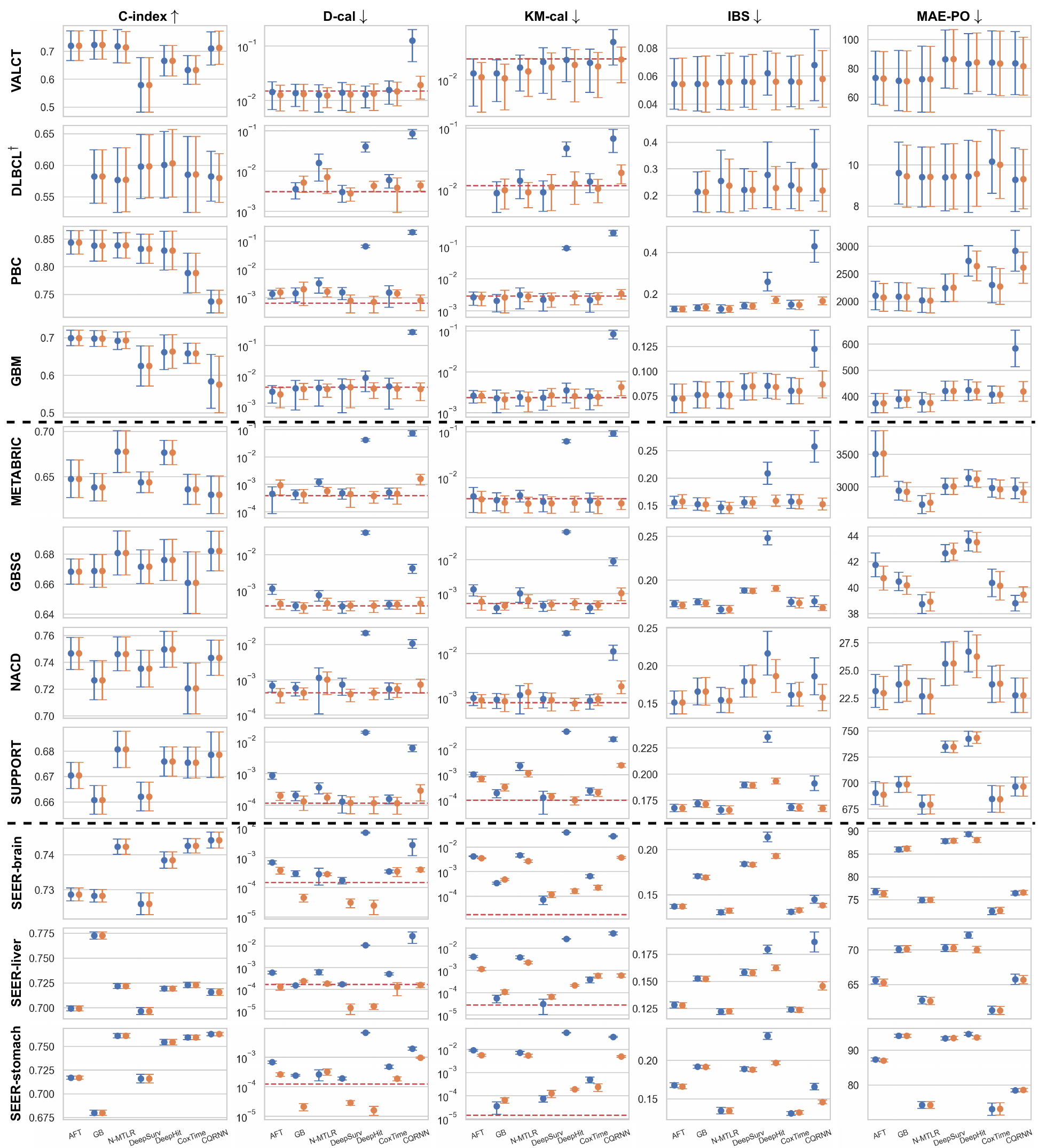}
    \caption{Complete empirical results. The error bars represent mean and 95\% CI over ten runs, with blue denoting the non-CSD baseline and orange for the CSD-version. The red dashed line represents the mean calibration performance for KM estimator, serving as an empirical lower-limit. A higher C-index score indicates better performance, whereas lower scores are preferable for other metrics.
    $^\dagger$AFT is fail to converge on \texttt{DLBCL} dataset due to the large feature-sample-ratio.}
    \label{fig:full_results}
\end{figure}

\subsection{Main Results: Compare CSD with Baseline Models}
In Figure~\ref{fig:full_results}, we showcase the comprehensive outcomes of our empirical evaluation for the CSD framework. 
This detailed figure organizes the data by dataset in each row, ordered by sample size. 
The datasets are segmented into small, medium, and large categories by the two black dashed lines for the sake of clarity. 
Evaluation metrics are depicted in columns, with the C-index employing a positive scoring rule -- where a higher value signifies superior performance, and the remaining metrics using negative scoring rules. 
Within the D-cal and KM-cal diagrams, red dashed lines highlight the empirical lower limits, which represent the mean calibration performance of a ``dummy'' KM model. 
In each subfigure, the performance of the standard non-CSD baseline is illustrated with blue bars, while the CSD-enhanced versions are depicted with orange bars, facilitating a clear comparison between the two.

Figure~\ref{fig:cal_slope} and~\ref{fig:km-cal_compare} present some qualitative calibration results for the CSD methods. Specifically, Figure~\ref{fig:cal_slope} shows the P-P plots for the non-CSD baselines and post-CSD models, and Figure~\ref{fig:km-cal_compare} shows the predictions of the average survival curves for the non-CSD baselines and the post-CSD models. Here, we use the DeepHit model as the baseline. We can see that for small and medium datasets, DeepHit tends to overestimate the risks, and our CSD process can calibrate this overestimation to almost perfectly calibrative estimates. For large datasets, DeepHit tends to underestimate the lower risks and overestimate the higher risks, and our CSD can also calibrate this estimation to almost perfect calibration.

\begin{figure}[ht]
    \centering
    \includegraphics[width=\textwidth]{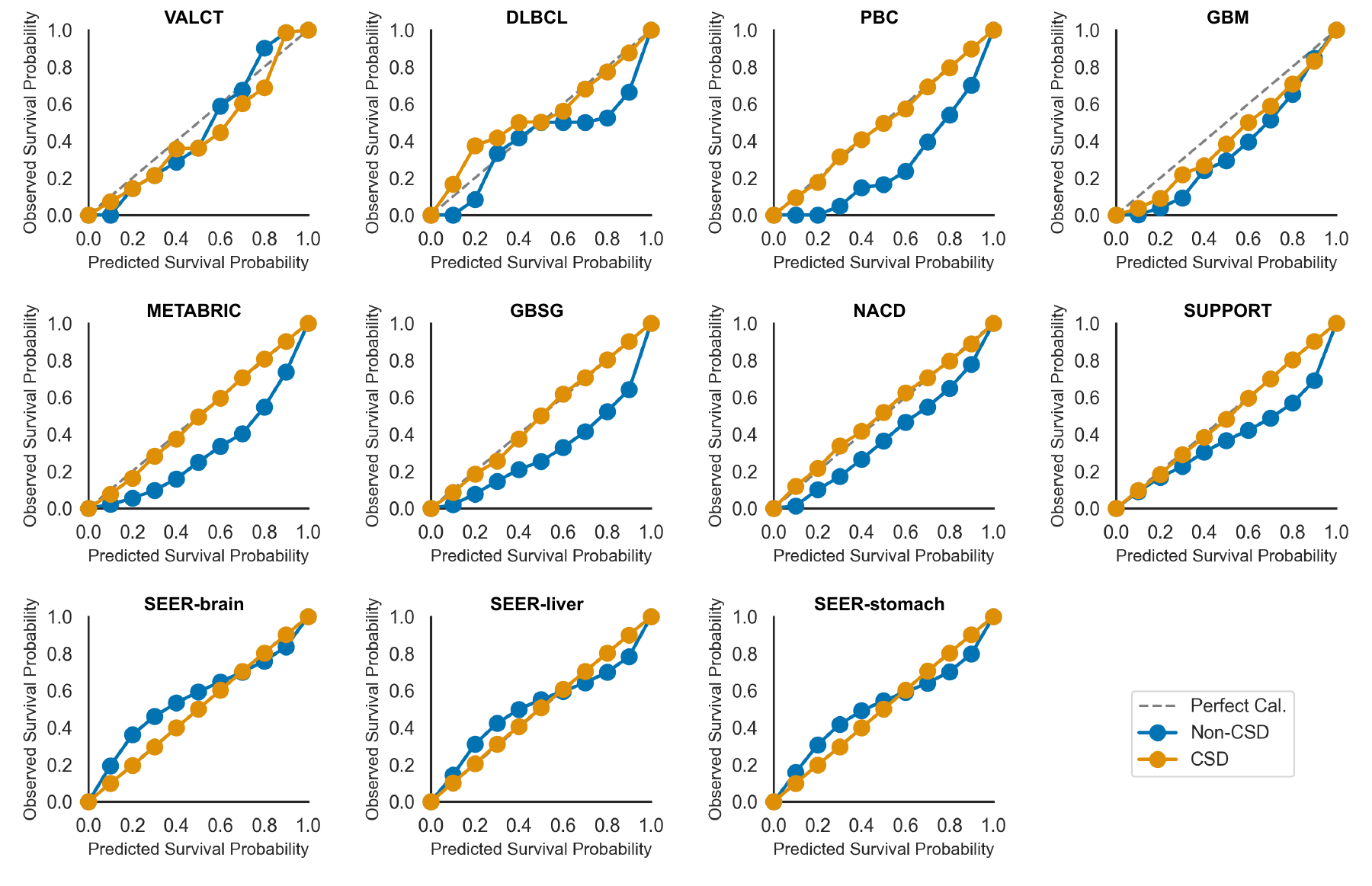}
    \caption{P-P plots assessing the calibration performance of the CSD and non-CSD methods. Here we use DeepHit model as the baseline (non-CSD) method, and the CSD method uses KM-sampling with 19 percentile levels. The dashed line represents perfect calibration.}
    \label{fig:cal_slope}
\end{figure}

\begin{figure}[ht]
    \centering
    \includegraphics[width=\textwidth]{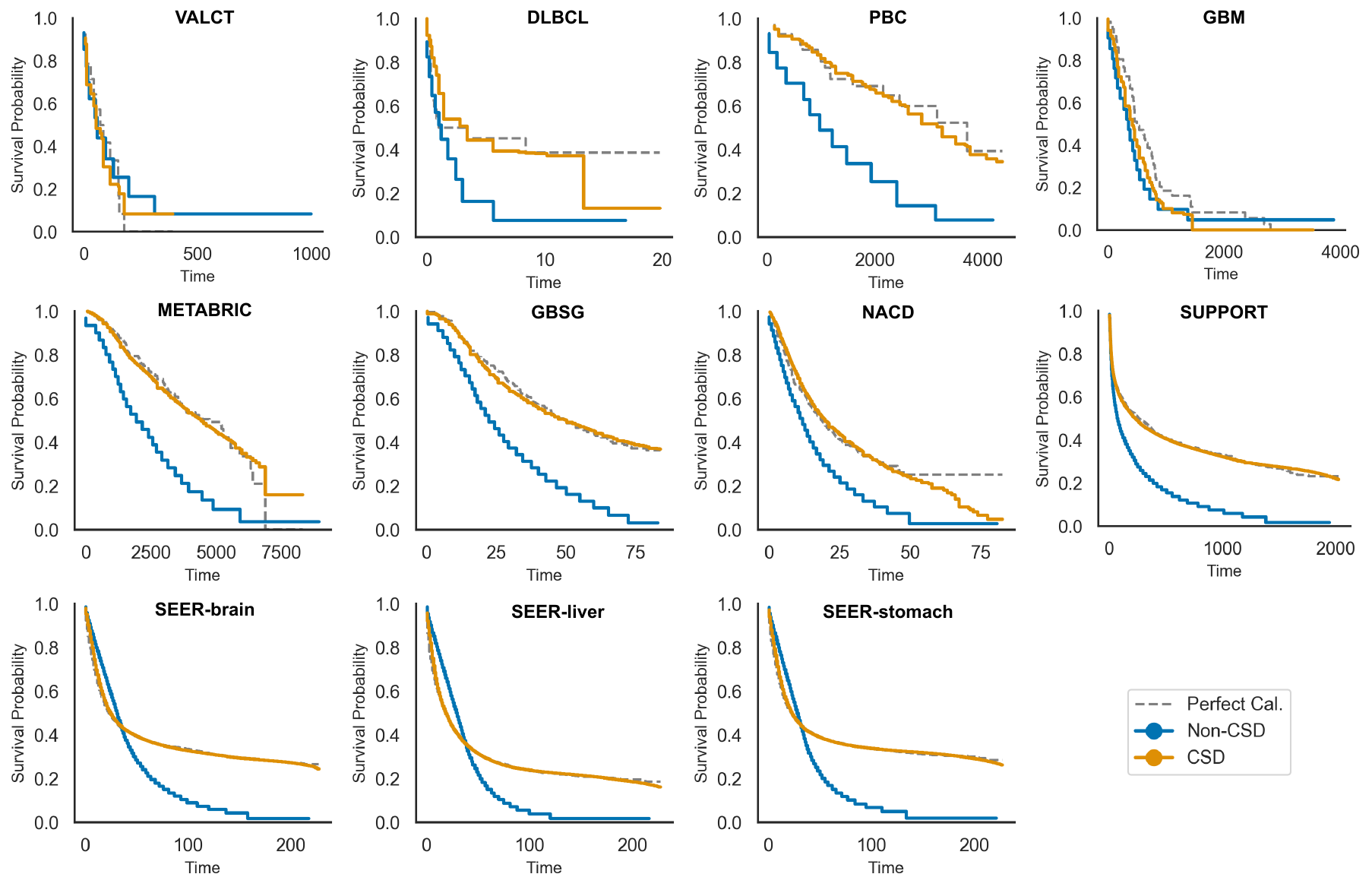}
    \caption{Compare the average predicted curve from the CSD and non-CSD methods with the true survival curve. Here we use DeepHit model as the baseline (non-CSD) method, and the CSD method uses KM-sampling with 19 percentile levels. The dashed line represents the Kaplan-Meier curve of the test set.}
    \label{fig:km-cal_compare}
\end{figure}

\subsection{Compare with Objective-Based Methods}
\label{appendix:objective-based_methods}

\begin{figure}[ht]
    \centering
    \includegraphics[width=0.9\textwidth]{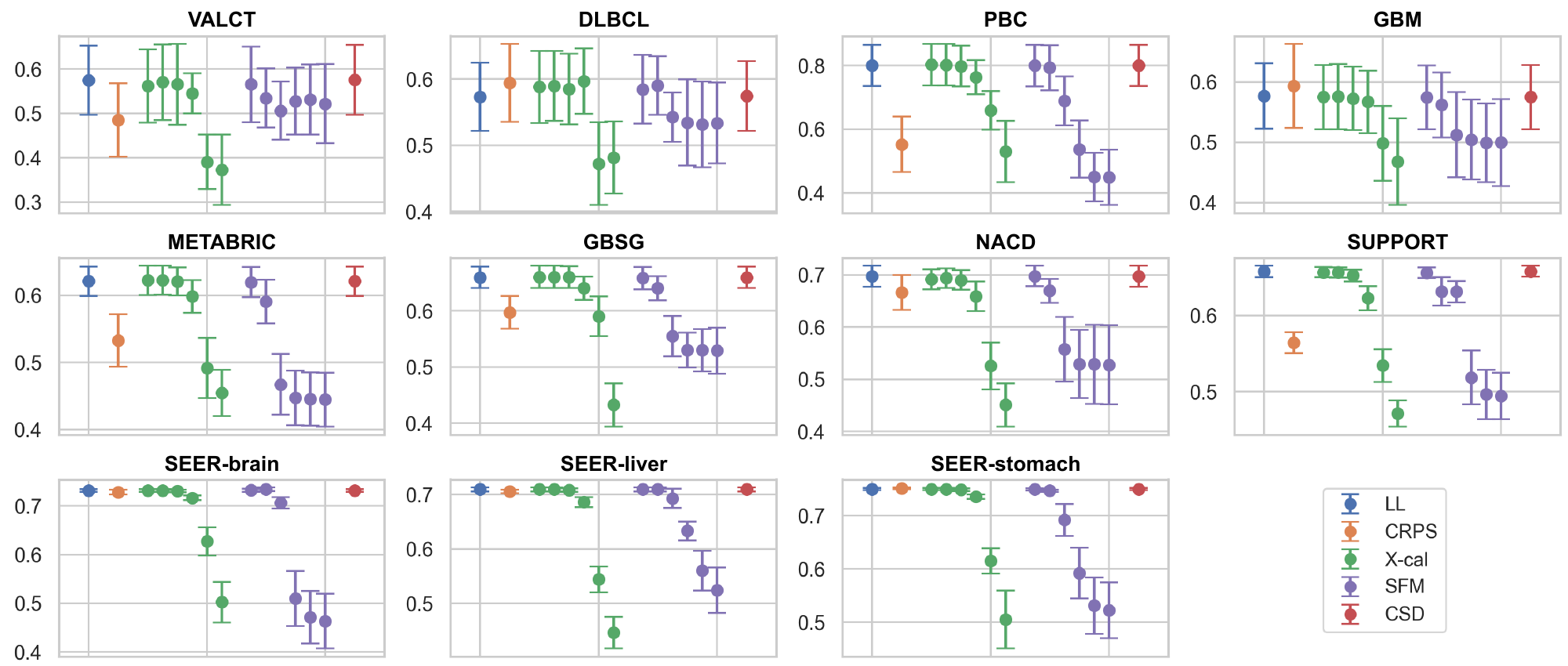}
    \caption{C-index comparison of CSD and objective-based methods. Higher values indicate superior performance. The baseline (blue bars) uses likelihood loss (LL), the yellow bars denote CRPS, the green bars denote X-cal, purple denotes SFM, and red bars denote CSD. For the X-cal and SFM methods, we gradually increase the weight for the calibration loss.}
    \label{fig:cindex_objective_methods}
\end{figure}

\begin{figure}[ht]
    \centering
    \includegraphics[width=0.9\textwidth]{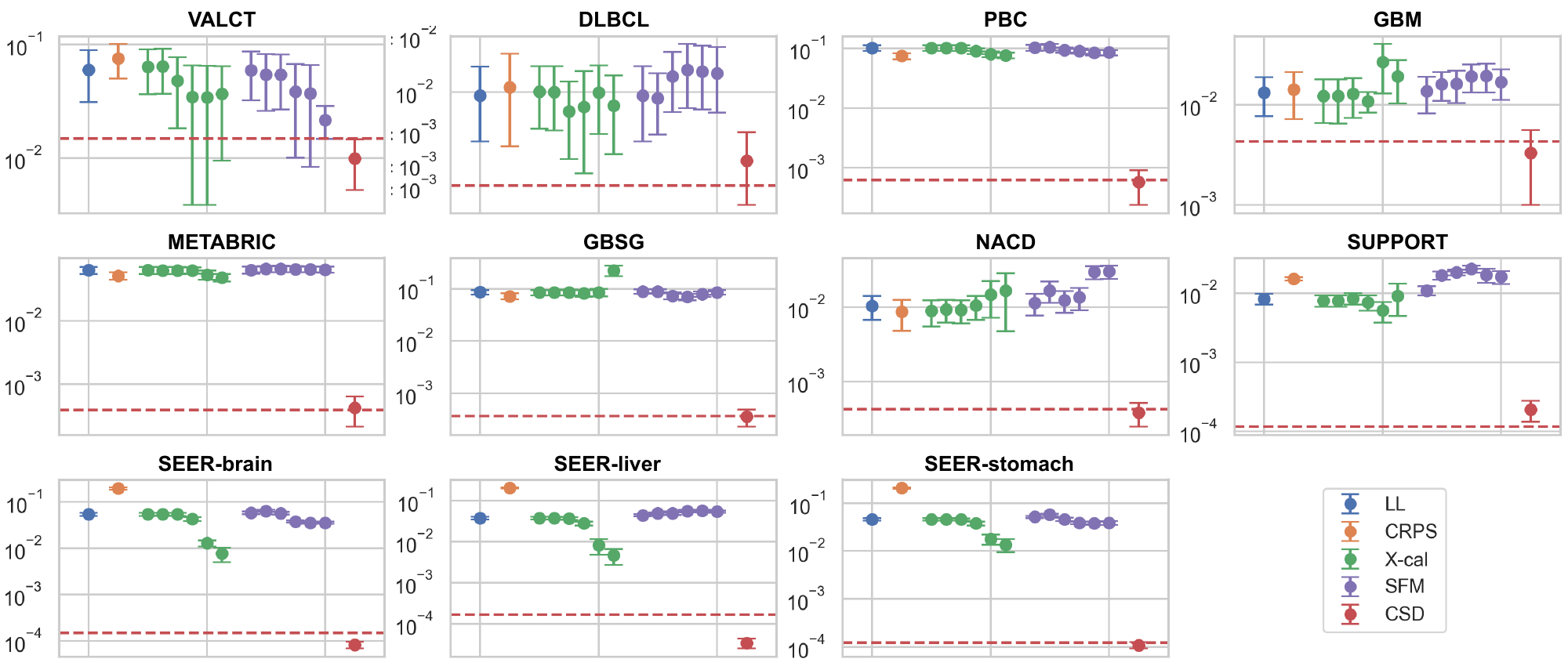}
    \caption{D-calibration comparison of CSD and objective-based methods. Lower values indicate superior performance. The baseline (blue bars) uses likelihood loss (LL), the yellow bars denote CRPS, the green bars denote X-cal, the purple bars denote SFM, and the red bars denote CSD. For the X-cal and SFM method, we gradually increase the weight for the calibration loss. The red dashed line represents the mean calibration performance of the KM estimator, serving as an empirical lower-limit.}
    \label{fig:dcal_objective_methods}
\end{figure}
\begin{figure}[ht]
    \centering
    \includegraphics[width=0.9\textwidth]{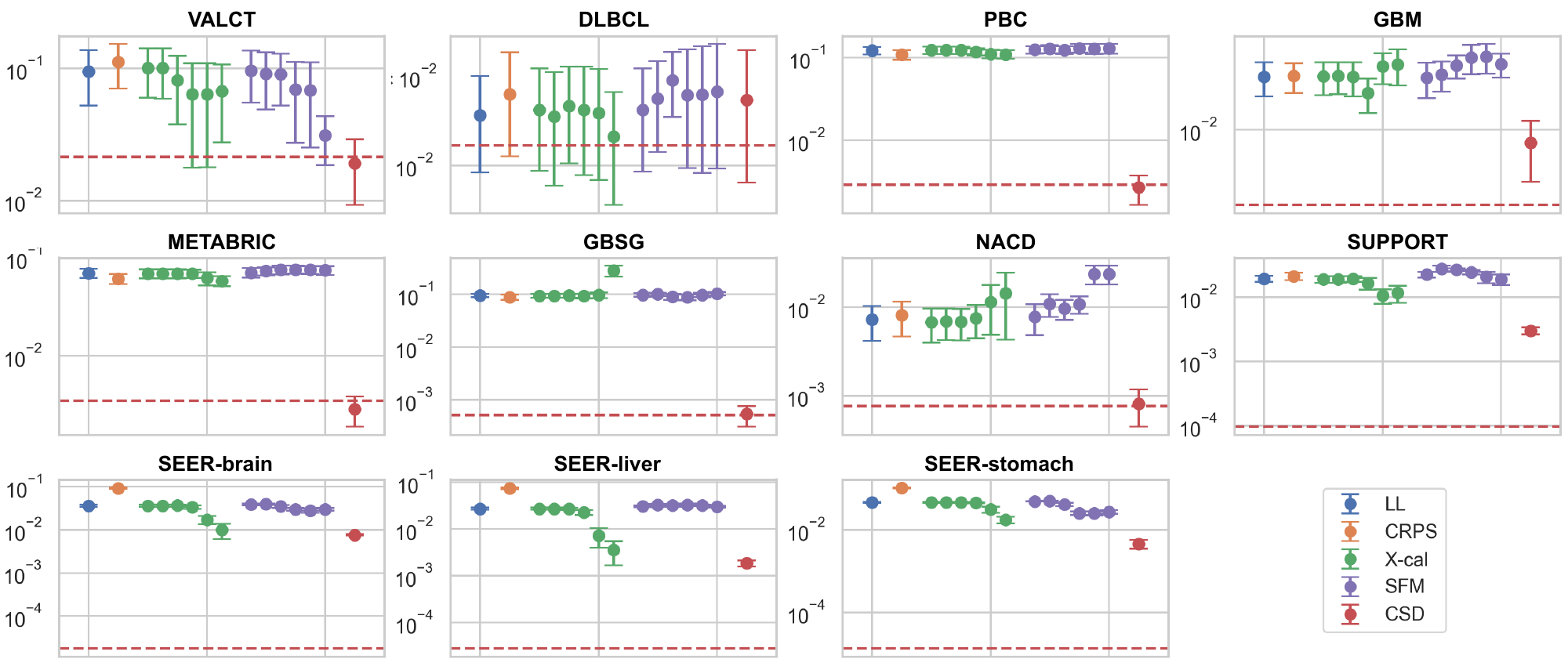}
    \caption{KM-calibration comparison of CSD and objective-based methods. Lower values indicate superior performance. The baseline (blue bars) uses likelihood loss (LL), the yellow bars denote CRPS, the green bars denote X-cal, the purple bars denote SFM, and the red bars denote CSD. For the X-cal and SFM method, we gradually increase the weight for the calibration loss. The red dashed line represents the mean calibration performance of the KM estimator, serving as an empirical lower-limit.
    }
    \label{fig:kmcal_objective_methods}
\end{figure}

To compare the proposed CSD with the objective-function based methods, we test the performance on a parametric log-normal model with neural networks (\emph{LogNormalNN}). \emph{LogNormalNN} uses two independent neural networks to predict the personalized mean and variance of the log-normal distribution, respectively.
This is the model used in both~\citet{avati2020countdown} and~\citet{goldstein2020x} as the baseline, which best shows the effectiveness of their methods.

Under \emph{conditional independent censoring} assumption, the baseline \emph{LogNormalNN} optimizes the log-likelihood function: 
\begin{align}
\label{eq:likelihhod}
\mathcal{L}_{\text{likelihood}} (\Model)
    &= \log\prod_{i = 1}^N \hat{f}_{\Model} (t = t_i \mid \Bfx{i})^{\delta_i} \hat{S}_\Model (t = t_i \mid \Bfx{i})^{1 - \delta_i} \notag \\
    &= \sum_{i = 1}^N \delta_i \times \log \hat{f}_{\Model} (t = t_i \mid \Bfx{i}) + (1 - \delta_i) \times \log \hat{S}_\Model (t = t_i \mid \Bfx{i}),
\end{align}
where $\hat{f}_{\Model} (t = t_i \mid \Bfx{i})$ is the predicted probability density function (PDF) which can be derived using ISD: $\hat{f}_{\Model} (t = t_i \mid \Bfx{i}) = \frac{- \partial \hat{S}_\Model (t = t_i \mid \Bfx{i})}{ \partial t}$.

Continuous Ranked Probability Score (CRPS) for survival is proposed by~\citet{avati2020countdown} to make sharp and calibrated ISD prediction. 
Specifically, they proposed to directly use the CRPS as the objective function:
\begin{align}
\mathcal{L}_{\text{CRPS}} (\Model)
    &= \sum_{i =1}^{N} \left( 
    \int_{t=0}^{t_i} \left(1 - \hat{S}_\Model (t  \mid \Bfx{i})\right)^2 dt + \delta_i \int_{t=t_i}^{\infty} \hat{S}_\Model (t  \mid \Bfx{i})^2 dt
    \right) .
\end{align}
As many other research points out~\cite{qi2023an}, the CRPS loss is essentially the IBS score without the IPCW weighting. 
Therefore, because IBS can be decomposed into an integrated 1-cal component,
optimizing CRPS should be able to lead to a calibrated prediction.

A follow-up study by~\citet{goldstein2020x} proposed a differentiable D-cal (X-cal) as part of the objective function with~\eqref{eq:likelihhod}.
For the mutually exclusive percentile intervals $\mathcal{I} = \{[0, 0.1], [0.1, 0.2], \ldots, [0.9, 1]\}$, the X-cal can be calculated using:
\begin{align*}
\mathcal{L}_{\text{X-cal}} (\Model)
    &= \sum_{I \in \mathcal{I}} \left( 
    \sum_{i =1}^{N} \zeta_\gamma \left(\hat{S}_\Model (t_i \mid \Bfx{i}); I \right) - |I|
    \right)^2 ,
\end{align*}
where
\begin{align*}
\zeta_\gamma (u; I = [a, b] ) = \frac{1}{1+\exp (- \gamma(u - a)(b - u))},
\end{align*}
is an approximation of the indicator function, which checks whether the predicted survival probability at the event time $\hat{S}_\Model (t_i \mid \Bfx{i})$ is inside the percentile interval $I$.

\citet{chapfuwa2020survival} proposed the survival function matching (SFM) loss as the penalty term for their survival models. 
The core component of the SFM loss is similar to the KM-cal (defined in Section~\ref{sec:metrics}), with a slight modification from the squared error to absolute error. 
Note that our implementation of SFM loss differs from the original by~\citet{chapfuwa2020survival} in two aspects: (1) The SFM loss was proposed in conjunction with a model that makes time prediction. In our baseline models, we need to predict survival distributions, which requires us to make substantial modifications to the SFM calculation. (2) The original was based on TensorFlow-based code, while our implementation is in PyTorch.

We reimplemented the baseline \emph{LogNormalNN} models with log-likelihood loss, CRPS loss, and X-cal loss, based on the code provided in~\cite{goldstein2020x}. 
We then test how our proposed CSD performs compared to these two objective function-based methods.

Figure~\ref{fig:cindex_objective_methods}, \ref{fig:dcal_objective_methods} and~\ref{fig:kmcal_objective_methods} provide a comparative analysis of the performance between the CSD and the objective-based methods, and various objective-based methods, focusing on three key metrics: C-index, D-cal, and KM-cal. 

The performance of different variations of the LogNormalNN model is depicted through color-coded bars:
\begin{itemize}
    \item The blue bars represent the \emph{LogNormalNN} model optimized using a likelihood function, illustrating the base performance without additional calibration or optimization criteria.
    \item The yellow bars indicate the \emph{LogNormalNN} model enhanced by the Continuous Ranked Probability Score (CRPS) loss function, showcasing the impact of optimizing for probabilistic forecasting accuracy.
    \item The green bars denote the \emph{LogNormalNN} model that integrates optimization for both the likelihood and an additional D-calibration penalty term (X-cal). Here, the emphasis is on the dynamic adjustment of the weight for the X-cal term, which is systematically increased from $10^0$ to $10^5$, illustrating the effect of calibration emphasis in model training.
    \item The purple bars denote the \emph{LogNormalNN} model that integrates optimization for both the likelihood and an additional KM-calibration penalty term (SFM). Here, the emphasis is on the dynamic adjustment of the weight for the SFM term, which is systematically increased from $10^0$ to $10^5$, illustrating the effect of calibration emphasis in model training.
    \item Lastly, the red bars illustrate the LogNormalNN model initially optimized using the likelihood function and subsequently adjusted using a post-processing CSD method, demonstrating the potential benefits of incorporating CSD adjustments after the initial model training.
\end{itemize}

\subsection{Ablation Studies}
\label{appendix:ablation}

\subsubsection{Ablation \#1: Handling Right-Censoring}

\begin{figure}
    \centering
    \includegraphics[width=\textwidth]{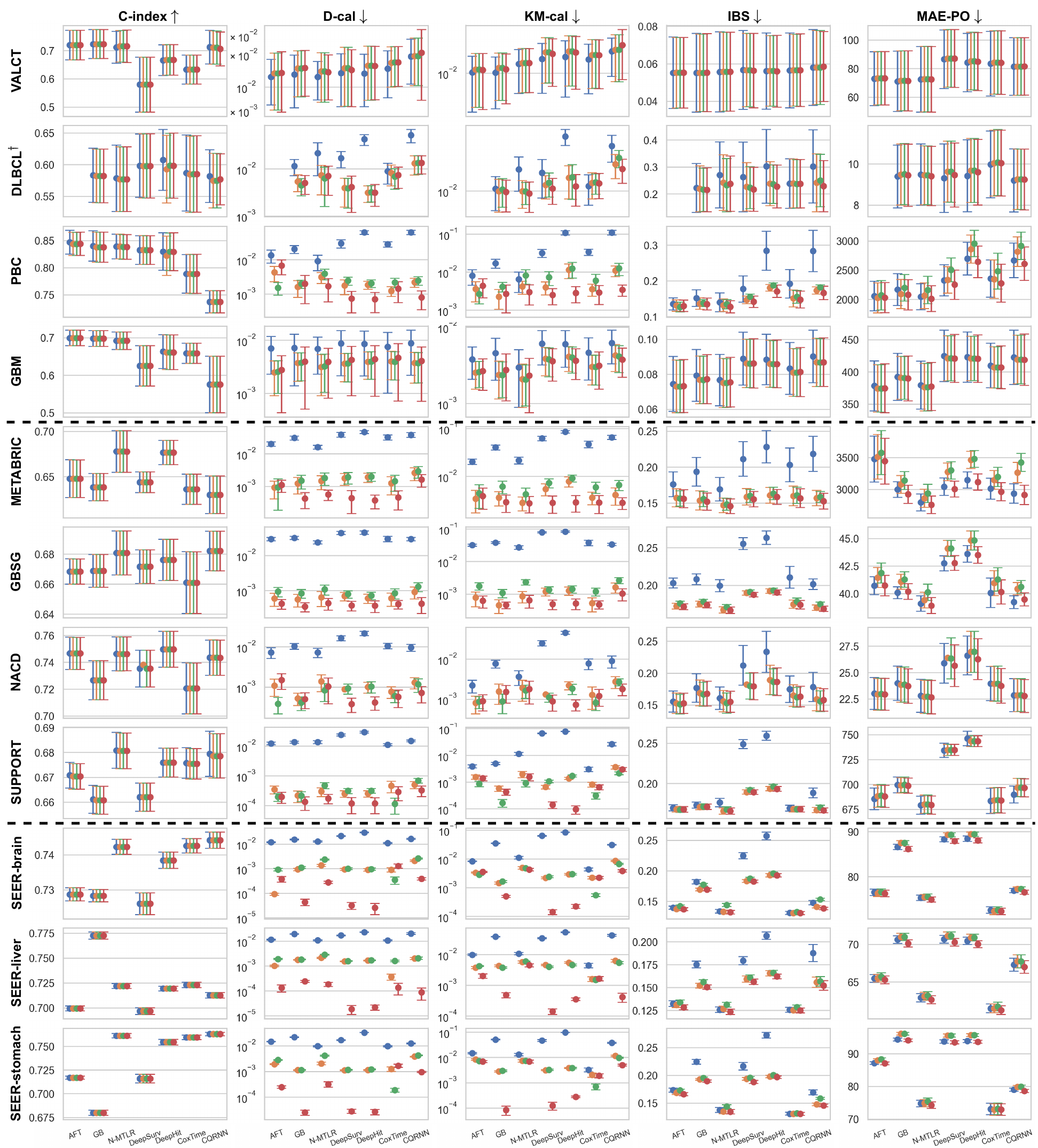}
    \caption{Comparison results for \textbf{ablation \#1: way of handling right-censoring}.
    The error bars represent mean and 95\% CI over ten runs, with blue denoting the uncensored method, yellow for margin, green for PO, and red for KM-sampling. 
    A higher C-index score indicates better performance, whereas lower scores are preferable for other metrics.
    $^\dagger$AFT is fail to converge on \texttt{DLBCL} dataset due to the large feature-sample-ratio.}
    \label{fig:ablation1}
\end{figure}

We present an ablation study in Figure~\ref{fig:ablation1} that investigates various approaches to managing censorship within the Conformalized Survival Distribution (CSD) framework.
These approaches include uncensored, margin, pseudo-observation (PO), and Kaplan-Meier (KM) sampling.
Similar to Figure~\ref{fig:full_results}, each row in this figure corresponds to a different dataset, arranged in ascending order of sample size, while the columns represent various evaluation metrics. 
The performance of each method is color-coded in the subfigures: uncensored (blue), margin (yellow), PO (green), and KM sampling (red).

\paragraph{TL;DR}
KM sampling (red bars) demonstrates superior calibration performance over the other methods (uncensored, margin, and PO), particularly in datasets of medium to large size.

The first columns of Figure~\ref{fig:ablation1} show negligible differences in the C-index between the non-CSD models and their CSD counterparts.
As to the calibration, the censoring method has difference impact for different types of datasets.
\begin{itemize}
    \item In datasets with very low censorship rates, such as \texttt{VALCT}, which contains only 9 censored subjects, the uncensored approach (blue bars) performs optimally. Indicating that omitting these few censored subjects from the CSD analysis has minimal effect on the overall data distribution.
    \item For other small datasets (\texttt{DLBCL}, \texttt{PBC}, and \texttt{GBM}), the performances of margin (yellow), PO (green), and KM sampling (red) are similar, likely due to inaccurate KM curve estimations from small sample sizes, affecting the KM sampling's ``best guess'' distribution.
    \item For medium and large datasets, KM sampling method (red bars) consistently (and significantly) works better than the others.
\end{itemize}
Regarding IBS performance, margin, PO, and KM sampling methods exhibit comparable outcomes within the 95\% confidence interval, yet all surpass the uncensored method. 
The MAE-PO shows no significant performance differences among the four methods.

\subsubsection{Ablation \#2: Conformal Set Splitting}

\begin{figure}
    \centering
    \includegraphics[width = \textwidth]{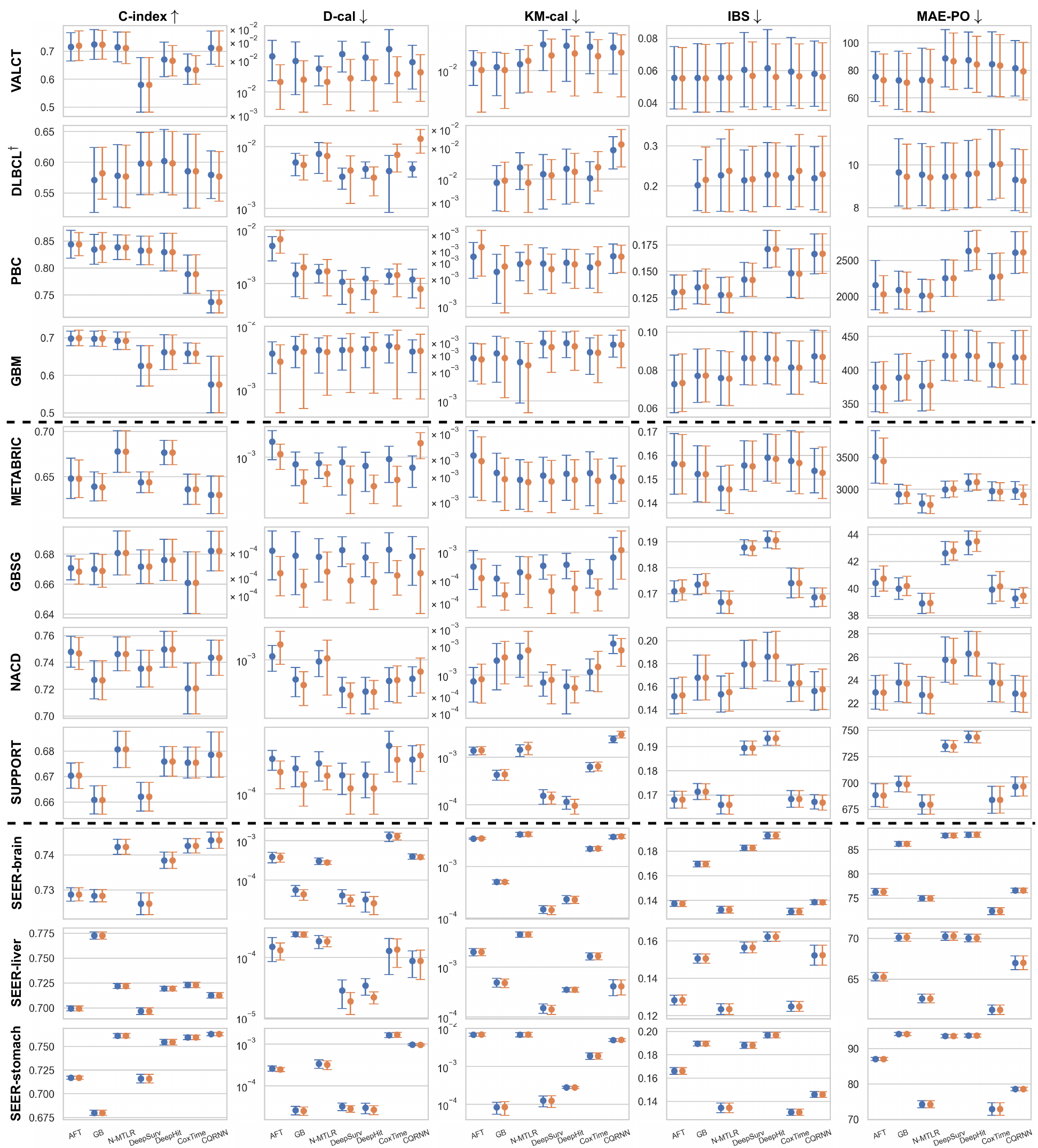}
    \caption{Comparison results for \textbf{ablation \#2: way of splitting conformal set}.
    The error bars represent mean and 95\% CI over ten runs, with blue denoting using solely the separate conformal set and yellow for combining the separate conformal set with the training set together for comformalization. 
    A higher C-index score indicates better performance, whereas lower scores are preferable for other metrics.
    $^\dagger$AFT is fail to converge on \texttt{DLBCL} dataset due to the large feature-sample-ratio.}
    \label{fig:ablation2}
\end{figure}

Figure~\ref{fig:ablation2} explores the performance of different strategies for dividing the conformal set.
We compare two policies (see also Section~\ref{sec:data_split}): (1) allocating a separate 10\% of the training set as the conformal set and (2) merging this separate conformal set with the training set. 
Each row in the figure corresponds to a dataset, arranged by sample size, with evaluation metrics displayed in columns. 
The performance of policy 1 is represented by blue bars, and policy 2 by yellow bars.

\paragraph{TL;DR}
For smaller or medium-sized datasets, policy 2 is more effective, while for larger datasets, policy 1 is preferred due to comparable performance and enhanced computational efficiency.

The figure indicates that the choice of data splitting policy has minimal effect on the C-index, IBS, and MAE-PO. 
However, for calibration performance in large datasets, using a separate validation set as the conformal set proves adequate. In such cases, a 10\% validation set yields at least 1000 conformal samples, which aligns with previous research findings~\cite{angelopoulos2023conformal}. For smaller or medium-sized datasets, where a 10\% validation set might result in only a few hundred samples, merging the validation and training sets for the conformal set (policy 2) is beneficial, significantly enhancing calibration performance. For large datasets, the calibration outcomes of both policies are similar, thus we recommend policy 1 (a separate conformal set) to save computational resources(faster and less memory usage).

\subsubsection{Ablation \#3: Number of Percentiles}

\begin{figure}
    \centering
    \includegraphics[width=\textwidth]{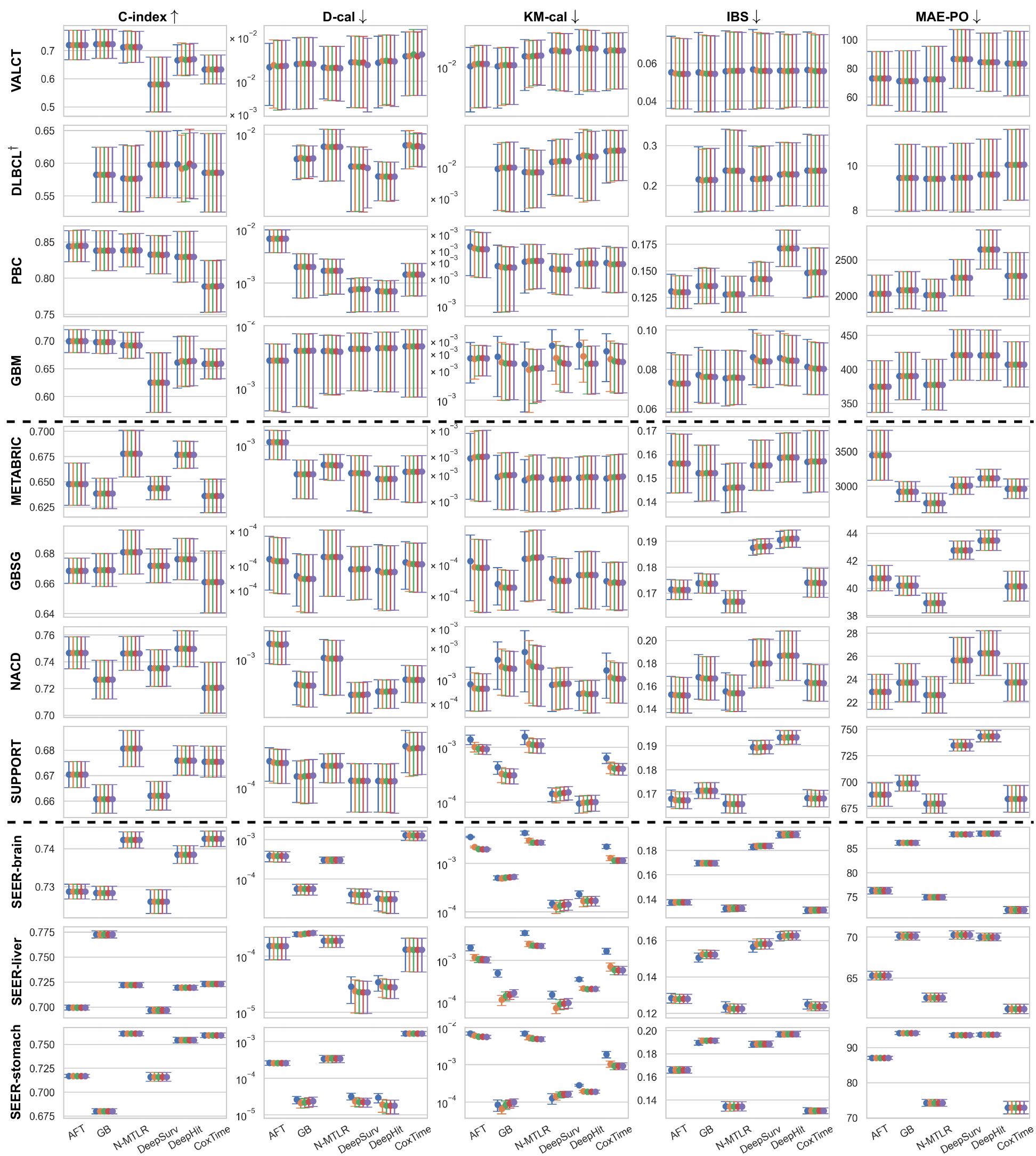}
    \caption{Comparison results for \textbf{ablation \#3: number of discretized percentiles}.
    The error bars represent mean and 95\% CI over ten runs, with blue denoting 9 percentiles, yellow for 19 percentiles, green for 39 percentiles, red for 49 percentiles, and purple for 99 percentiles. 
    A higher C-index score indicates better performance, whereas lower scores are preferable for other metrics.
    $^\dagger$AFT is fail to converge on \texttt{DLBCL} dataset due to the large feature-sample-ratio.}
    \label{fig:ablation3}
\end{figure}

In Figure~\ref{fig:ablation3}, we examine the effect of varying the number of discretized percentiles on CSD performance.
We consider several options for $\mathcal{P}$: 9 percentiles (10\% to 90\% at 10\% intervals), 19 percentiles (5\% to 95\% at 5\% intervals), 39 percentiles (2.5\% to 97.5\% at 2.5\% intervals), 49 percentiles (2\% to 98\% at 2\% intervals), and 99 percentiles (1\% to 99\% at 1\% intervals). 
This figure also presents one dataset in each row, ordered by sample size, and evaluation metrics are depicted in columns.
The performance of different numbers of percentile is color-coded in the subfigures: 9 (blue), 19 (yellow), 39 (green), 49 (red), and 99 (purple).

\paragraph{TL;DR}
The number of percentiles has minimal impact on the performance.

The choice of percentile number has little to no effect on survival metric performance. 
However, for large datasets (\texttt{SEER-brain}, \texttt{SEER-liver}, and \texttt{SEER-stomach}), KM-cal results (third column) indicate that using only 9 percentiles leads to significantly poorer performance, likely due to information loss with such coarse discretization. 
For practical applications, selecting either 9 or 19 percentiles is advisable to maintain a balance between calibration accuracy and computational demand.

\subsection{Computational Analysis}
\label{appendix:computational_analysis}
One limitation of the CSD method is that it cannot perform the conformalization step in batch mode, which is computationally inefficient.

Specifically, the quantile operation in~\eqref{eq:csd_adjust} requires storing the conformity score for all data in the conformal set in advance, followed by computing the quantile amount on a substantial array. This will introduce two complexities: space complexity and time complexity.

Although there are ways to perform quantile calculations in batch modes, this will only approximate the true quantile values. To achieve precise quantile values, it is necessary to store the conformity scores beforehand, which incurs complexity overhead.

\subsubsection{Space Complexity}

To simplify notation, let $|\Data_{\text{con}}|$ be the number of subjects in the conformal datasets, $\mathcal{P}= \{\rho_1, \rho_2, \ldots \}$ be the predefined discretized percentiles, so that $|\mathcal{P}|$ be the number of predefined percentiles, and $R$ be the repeat parameter for KM-sampling. Table~\ref{tab:space_complex} presents the space complexity for storing the conformity scores. 
Compared with the uncensored method, the KM-sampling method will asymptotically cause around $R$ times memory overhead.

\begin{table}[ht]
\centering
\caption{The space complexity comparison between Conformal Quantile Regression (CQR), CSD-uncensored, CSD-margin, CSD-PO, and CSD-KM-sampling. $^\dagger$This assumes the number of uncensored subjects in the datasets is $O(|\Data_{\text{con}}|)$. }
\label{tab:space_complex}
\begin{tabular}{lccccc}
\toprule
 & CQR  & CSD-uncensored & CSD-margin & CSD-PO & CSD-KM-sampling \\   \midrule
Space Complexity & $O(|\Data_{\text{con}}|)$ & $O(|\Data_{\text{con}}| \cdot |\mathcal{P}|)^\dagger$         & $O(|\Data_{\text{con}}|\cdot|\mathcal{P}|)$      & $O(|\Data_{\text{con}}|\cdot|\mathcal{P}|)$  & $O(|\Data_{\text{con}}| \cdot |\mathcal{P}| \cdot R)$ \\
\bottomrule
\end{tabular}
\end{table}

\subsubsection{Time Complexity}

Figure~\ref{fig:time_complex} and~\ref{fig:time_complex_relative} empirically compare the running times of the CSD method versus various baseline methods. 
Each value (dot) in Figure~\ref{fig:time_complex} represents an average extra running time of the method, over the associated non-CSD baselines, over 10 random splits. 
Each value (dot) in Figure~\ref{fig:time_complex_relative} represents an average relative running time of the method, compared with the associated non-CSD baselines, over 10 random splits. 
We see that the running time of the CSD framework is independent of the baseline algorithms. 
For survival algorithms that require less training time (\eg AFT), the CSD post-process with KM-sampling may significantly amplify the running time ($>100$ times compared with non-CSD baselines for \text{SEER} datasets, and $\sim 16$ times compared with CSD-uncensored for \text{SEER} datasets). 
For survival algorithms that require long training times (\eg GB, CQRNN, etc.), the relative time for the CSD post-process can sometimes be ignored.

\begin{figure}[ht]
    \centering
    \includegraphics[width=\textwidth]{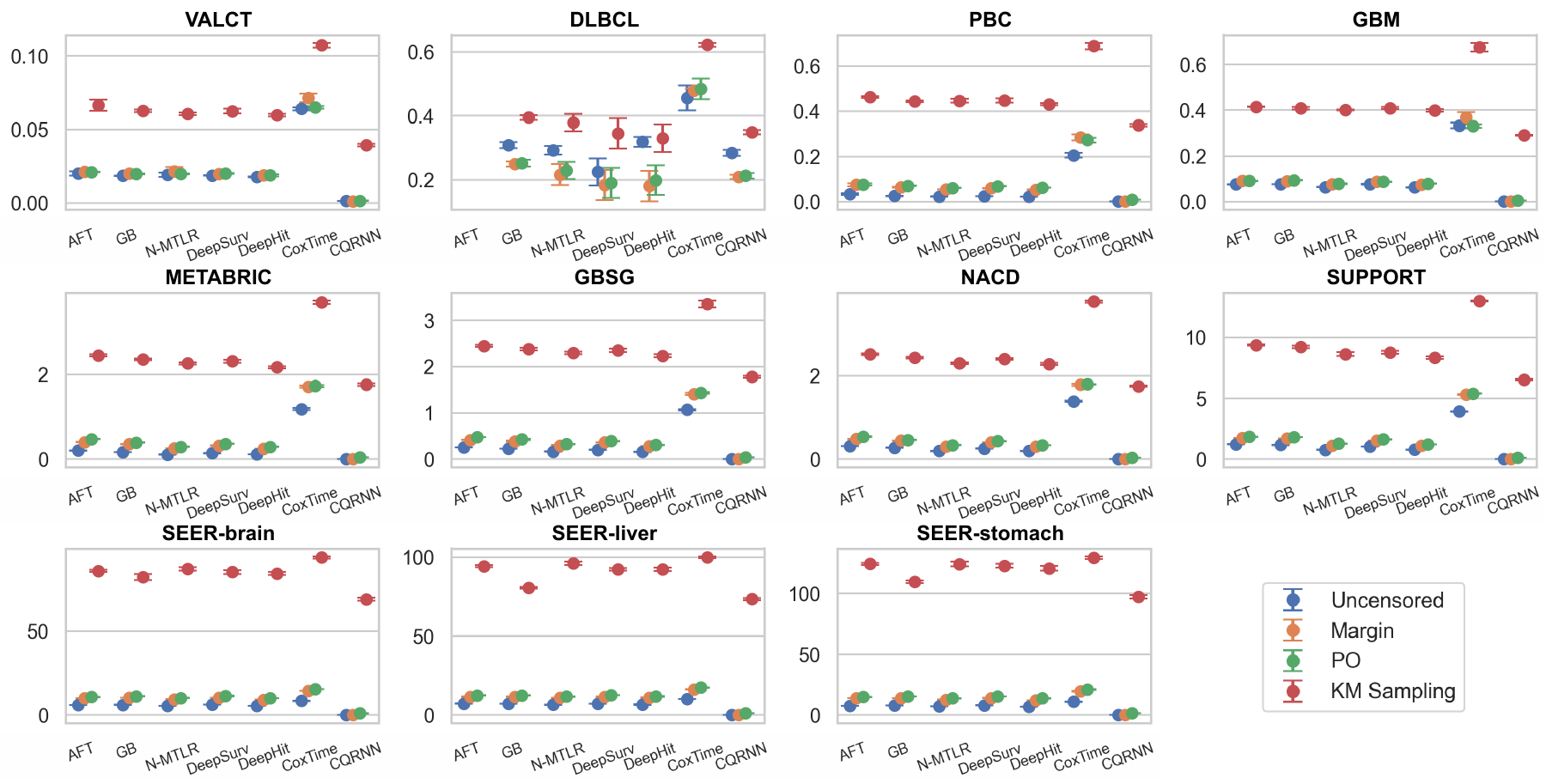}
    \caption{Extra running time (mean with 95\% CI) of the CSD method (with four different methods for handling censorship) on top of the baseline methods.}
    \label{fig:time_complex}
\end{figure}

\begin{figure}[ht]
    \centering
    \includegraphics[width=\textwidth]{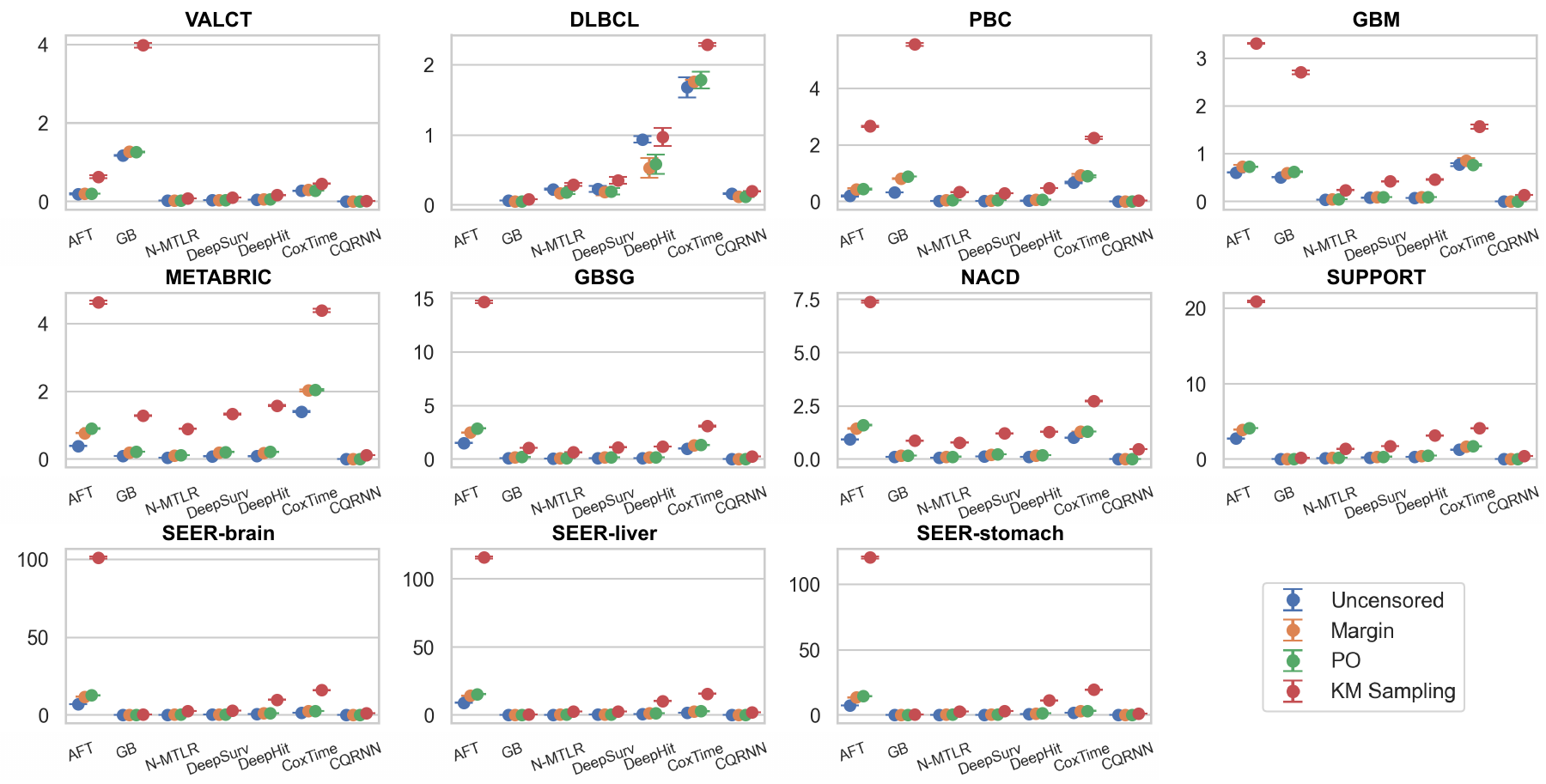}
    \caption{Relative running time (mean with 95\% CI) of the CSD method (with four different methods for handling censorship) on top of the baseline methods.}
    \label{fig:time_complex_relative}
\end{figure}


\end{document}